\documentclass{article}

\PassOptionsToPackage{numbers}{natbib} % plain.bst is numeric; avoids natbib author-year error
\usepackage[main, final]{neurips_2026}
\usepackage{silence}
\usepackage[utf8]{inputenc} % allow utf-8 input
\usepackage[T1]{fontenc}    % use 8-bit T1 fonts
\usepackage{url}            % simple URL typesetting
\usepackage{booktabs}       % professional-quality tables
\usepackage{amsfonts}       % blackboard math symbols
\usepackage{nicefrac}       % compact symbols for 1/2, etc.
\usepackage{microtype}      % microtypography
\usepackage[table]{xcolor}  % colors and colored table cells

\usepackage{graphicx}
\usepackage{calc} 
\usepackage{subfigure}
\usepackage{multirow}
\usepackage{enumitem}
\usepackage{tikz}
\usepackage{amsmath}
\usepackage{amssymb}
\usepackage{mathtools}
\usepackage{amsthm}
\usepackage{verbatim}
\usepackage{wrapfig}
\usepackage{algorithm}
\usepackage{algorithmic}
\usepackage{placeins}  % 프리앰블에 추가
\usepackage{makecell}
\usepackage{adjustbox}
\usepackage{capt-of}   % \captionof{table} inside the Figure 3 float
\usepackage[hidelinks]{hyperref}
\theoremstyle{plain}
\newtheorem{theorem}{Theorem}[section]
\newtheorem{proposition}[theorem]{Proposition}

\theoremstyle{definition}

\theoremstyle{remark}

\newcommand{\bc}{\begin{center}}
	\newcommand{\ec}{\end{center}}
\newcommand{\be}{\begin{equation}}
	\newcommand{\ee}{\end{equation}}
\newcommand{\ba}{\begin{array}}
	\newcommand{\ea}{\end{array}}
\newcommand{\bean}{\setlength\arraycolsep{1pt}\begin{eqnarray*}}
	\newcommand{\eean}{\end{eqnarray*}}
\newcommand{\bea}{\setlength\arraycolsep{1pt}\begin{eqnarray}}
	\newcommand{\eea}{\end{eqnarray}}
\newcommand{\ben}{\begin{enumerate}}
	\newcommand{\een}{\end{enumerate}}
\newcommand{\bed}{\begin{itemize}}
	\newcommand{\eed}{\end{itemize}}

\DeclareMathOperator*{\argmin}{argmin}

\title{Optimal Transport Reweighting for Robust Learning under Spurious Correlations and Label Noise}

\author{%
  Sung Ho Jo \quad Seonghwi Kim \quad Wonsang Yun \quad
  Minwoo Chae\thanks{Corresponding author.} \\
  Pohang University of Science and Technology (POSTECH) \\
  \texttt{\{tjdgh1813, kshwi, wonsang, mchae\}@postech.ac.kr} \\
}

\begin{document}

\maketitle

\begin{abstract}
Machine learning models often suffer performance degradation under subpopulation shift, particularly when spurious correlations cause models to rely on shortcut features that fail to generalize across subgroups. A recent line of work mitigates this issue by using loss-based signals to identify informative samples, but these signals can become severely distorted under label noise: mislabeled samples may also incur large losses and contaminate subsequent reweighting or retraining. Despite its practical importance, this intersection remains largely underexplored. We propose POTER, a reweighting framework based on optimal transport that derives sample importance from the transport geometry between the training distribution and a reference distribution constructed from limited validation group annotations. By measuring alignment at the individual-sample level rather than relying on loss, POTER downweights mislabeled or strongly bias-aligned samples while assigning higher importance to samples better aligned with the reference distribution. In addition, POTER requires only a single ERM training stage, moving beyond the retraining paradigm common in recent work. Across standard benchmarks and noisy-label settings, POTER achieves state-of-the-art worst-group accuracy, including cases where label corruption is concentrated within minority subgroups.
\end{abstract}

\section{Introduction}

Empirical risk minimization (ERM) assumes that training and test data are drawn from the same distribution, and has achieved strong average performance when this assumption holds. In practice, however, models are often deployed under distribution shifts. A particularly challenging instance is subpopulation shift, in which the proportions of subpopulations differ between training and test distributions. This becomes especially problematic when certain attributes are spuriously correlated with the target label in the training data, causing models to rely on shortcut features that fail to generalize across subpopulations \cite{sagawa2019distributionally}. For example, in the Waterbirds dataset, bird labels are strongly correlated with backgrounds, leading models to exploit background cues rather than bird-specific features. Performance consequently deteriorates on minority subgroups where such correlations no longer hold—a vulnerability with serious implications in high-stakes domains such as healthcare and fairness-sensitive decision-making \cite{obermeyer2019dissecting, banerjee2023shortcuts, hashimoto2018fairness}.

To mitigate spurious correlations, many existing methods rely on group annotations to identify underperforming subpopulations and promote subgroup robustness during training \cite{sagawa2019distributionally, kirichenkolast, jo2025mitigating}. In realistic settings, however, group annotations for the full training set are often unavailable, with such information available only for a limited validation set. To address this limitation, a prominent line of work relies on sample-wise proxy signals—such as per-sample loss or misclassification indicators—from a preliminary model to identify minority or bias-conflicting samples, which are then upweighted or used in last-layer retraining pipelines \cite{liu2021just, qiu2023simple, stromberg2024robustness, labonte2024towards}. Bias-conflicting samples refer to those that do not conform to the dominant spurious correlation in the training data (e.g., waterbirds on land backgrounds in the Waterbirds dataset). The central weakness of this line of work is that it hinges on high-loss samples being informative. That premise is fragile even in clean settings, since loss-based signals depend on the learned model and optimization trajectory, making the selected samples inconsistent across training procedures.

Under label noise, this fragility becomes especially severe. Mislabeled samples also tend to incur large training losses, so high-loss-based criteria systematically select them along with genuinely informative bias-conflicting examples. As a result, the selected subset becomes contaminated, and the subsequent reweighting or retraining step can no longer reliably emphasize bias-conflicting samples. Despite its practical importance, the intersection of spurious correlations and label noise remains largely overlooked. One recent approach is \cite{NEURIPS2024_fba9133d}, which proposes label propagation prior to last-layer retraining as a remedy and attains strong subpopulation robustness under symmetric label noise, where each training label is independently flipped to the other class with the same probability. However, this robustness is established when the model used in the subsequent reweighting or retraining stage is trained on clean data, and the gains become substantially smaller once label noise affects the full training pipeline. Moreover, it degrades further when label noise is concentrated within minority subgroups—a setting we term \textit{subgroup-concentrated label noise}—exposing a previously overlooked failure mode.

To address the fragility of proxy-based sample selection, we introduce POTER (Potentials from Optimal Transport for Empirical Reweighting), an OT-based framework that constructs a reweighted empirical training distribution using only limited group annotations from a validation set. Rather than identifying informative samples using a preliminary model and then upweighting or retraining on the selected subset, our method derives sample importance directly from the OT geometry between the training distribution and a constructed reference distribution. This removes optimization-dependent sample selection entirely and does not rely on a clean-data-trained preliminary model. Instead, sample importance is determined by alignment with the reference distribution, so examples that are poorly aligned with the target distribution—including those affected by label noise—are assigned lower importance. To our knowledge, this is the first OT-based reweighting framework for spurious correlation mitigation. More fundamentally, the proposed framework moves beyond the retraining-based paradigm that has largely dominated this field: unlike existing approaches that rely on full-model or last-layer retraining, it requires only a single standard ERM training stage, with OT dual potentials computed within seconds even on the largest benchmark we consider. Despite this simplicity, POTER achieves state-of-the-art performance on standard spurious correlation benchmarks and substantially outperforms prior methods under both symmetric and subgroup-concentrated label noise, without sacrificing clean-setting performance.

Our main contributions are as follows:

\begin{itemize}[leftmargin=*, itemsep=0.5em]
	\item We propose POTER, a novel OT-based reweighting framework for spurious correlation mitigation, deriving sample importance from OT dual potentials using only limited validation-set group annotations, without proxy-based sample selection or preliminary downstream-task training.
	\item We advance beyond recently established last-layer retraining pipelines, requiring only a single standard ERM training stage with OT computation adding negligible overhead.
    \item We demonstrate state-of-the-art worst-group accuracy on standard benchmarks and substantially outperform existing methods under symmetric label noise across a wide range of noise levels.
    \item We identify subgroup-concentrated label noise as a previously overlooked failure mode, and show that our approach remains robust in this regime, where even a method explicitly designed to address label noise under spurious correlations degrades substantially.
\end{itemize}

\section{Related Work}

\paragraph{Spurious Correlation Mitigation}

A large body of work has studied robust learning under spurious correlations. Early approaches typically assume access to group annotations for the full training set and use them to guide robust learning \cite{arjovsky2019invariant, sagawa2019distributionally, krueger2021out, yao2022improving, deng2024robust}. To reduce reliance on full training-set group labels, later methods use limited supervision to infer pseudo groups or spurious attributes and combine them with existing robust learning methods like Group DRO, but this requires an additional group-inference stage and incurs extra computational cost \cite{nam2022spread, ghosal2023distributionally, hanimproving}. More recently, retraining-based pipelines—especially those based on last-layer retraining—have become a dominant design choice, using proxy signals from a preliminary model to select or reweight samples before retraining or adaptation \cite{liu2021just, qiu2023simple, stromberg2024robustness, labonte2024towards, ghaznavi2025exploiting, gourabathina2026robustness}. Although effective, these methods still depend on additional retraining and on the quality of the preliminary model used to generate such signals. Our work pushes this progression one step further by directly constructing a reweighted empirical training distribution, without requiring any retraining stage. Beyond these established directions, recent work has also begun to examine new failure modes within the spurious correlation framework, including settings with multiple spurious features \cite{paranjapeagro, kim2024improving}, weak spurious signals \cite{mulchandanisevering}, distributional shifts within minority subgroups \cite{jo2025mitigating}, and settings where some subgroups are entirely absent from the training distribution \cite{seonghwi2026}. This broader trend underscores the need to study additional sources of failure beyond standard benchmark assumptions.

\paragraph{Spurious Correlation under Label Noise}

While recent work has increasingly examined diverse failure modes within the spurious correlation framework, label noise—a classical challenge in machine learning—has received comparatively limited attention. More broadly, prior studies \cite{oh2022improving, qiaounderstanding, NEURIPS2024_fba9133d} suggest that many existing spurious correlation methods—especially those relying on loss- or error-based signals—can suffer severe performance degradation under label noise, regardless of whether group information is available.
Among the few works that directly address this intersection, END \cite{oh2022improving} uses predictive uncertainty to identify and oversample samples that are likely to be free of spurious cues, but requires two rounds of full-model training and tends to yield weaker clean-setting performance. More recently, a kNN-based label-correction method \cite{NEURIPS2024_fba9133d} improves last-layer retraining under symmetric label noise. However, its reliance on local neighborhood structure makes it sensitive to subgroup-concentrated label noise. We show that even in noise regimes where this method appears to perform well, worst-group accuracy can deteriorate substantially when label noise is concentrated within minority subgroups, even if the overall noise level remains moderate.

\paragraph{Optimal Transport and Data Valuation}
Optimal transport provides a principled framework for comparing and aligning probability distributions, and has found broad application in machine learning, particularly in robust learning under distribution shift \cite{kuhn2019wasserstein, blanchet2022distributionally, bai2024wasserstein}.
More recently, OT dual potentials have been leveraged as theoretically grounded signals for data valuation \cite{just2023lava}, with subsequent work extending this framework to federated learning settings \cite{li2024data} and scaling it to larger datasets \cite{kessler2024sava}. Our work builds on this perspective, using OT dual potentials not only for data valuation, but also to derive sample importance weights for model learning. To our knowledge, this is the first work to introduce OT-based reweighting for spurious correlation mitigation, enabling us to construct a reweighted empirical training distribution using only limited group annotations from a validation set—without proxy-based sample selection or any retraining stage.

\section{Preliminaries}

\paragraph{Problem Setup.}
\label{sec:preliminaries}
We consider a supervised classification problem with input $X \in \mathcal{X}$, clean label $Y \in \mathcal{Y}$, and subgroup (or subpopulation) index $G \in \mathcal{G}$, where $K:=|\mathcal{Y}|$ denotes the number of classes. Let $P_{\mathrm{train}}$ and $P_{\mathrm{test}}$ denote the training and test distributions over $(X,Y,G)$, respectively.
We assume a subpopulation shift where the marginal distributions of \(G\) under \(P_{\mathrm{train}}\) and \(P_{\mathrm{test}}\) differ, while the conditional distribution of \((X,Y)\) given \(G\) remains invariant. Under subgroup-dependent label noise, we observe a corrupted label $\tilde Y$, instead of $Y$, governed by the conditional probabilities
\[
P_{\mathrm{train}}(\tilde Y=\tilde y \mid Y=y,\,G=g)
=
(1-\eta_g)\mathbf{1}[\tilde y=y]
+
\frac{\eta_g}{K-1}\mathbf{1}[\tilde y\neq y],
\qquad g\in\mathcal G,
\]
where $\eta_g\in[0,1)$ denotes the noise rate for subgroup $g$. Symmetric label noise corresponds to the special case $\eta_g=\eta$ for all $g\in\mathcal G$. We define subgroup-concentrated label noise as the scenario in which corruption is restricted to a specific subset $S\subseteq\mathcal G$ of subgroups. In our experiments, we specifically consider the case where $\eta_g=\eta$ for $g\in S$ and $\eta_g=0$ otherwise.

We are given a training set $D_t=\{(x_i,\tilde y_i)\}_{i=1}^n$ consisting of independent samples from the marginal distribution of $(X,\tilde Y)$. Crucially, the subgroup indices $G$ are latent and remain unobserved throughout training. We are also given a limited validation set $D_v=\{(x_j',y_j',g_j')\}_{j=1}^m$ sampled from $P_{\mathrm{test}}$, with clean labels and subgroup annotations.

Our objective is to learn a classifier that remains robust under subpopulation shift even when the training labels are noisy. We evaluate performance on the test distribution using worst-group accuracy, i.e., the minimum accuracy over all subgroups, which effectively captures robustness to subgroup-proportion shifts, including extreme cases where test mass concentrates on a single subgroup. For the distributional formulation used later, we represent the observed training data and validation data by their empirical distributions, $\mu_t=n^{-1}\sum_{i=1}^n \delta_{(x_i,\tilde y_i)}$ and $\mu_v=m^{-1}\sum_{j=1}^m \delta_{(x_j',y_j')}$, respectively, where $\delta_z$ denotes the Dirac measure at $z$.

\paragraph{Optimal Transport.}
Let $\mathcal{Z}=\mathcal{X}\times\mathcal{Y}$. Given empirical distributions $\mu_t$ and $\mu_v$ and a cost function $c:\mathcal{Z}\times\mathcal{Z}\to\mathbb{R}_+$, we define the optimal transport value between $\mu_t$ and $\mu_v$ as the optimal value of the Kantorovich problem \cite{kantorovich1942translocation}:

\begin{equation}
\mathcal{T}(\mu_t,\mu_v)
:=
\inf_{\pi\in\Pi(\mu_t,\mu_v)}
\int_{\mathcal{Z}\times\mathcal{Z}} c(z,z')\,d\pi(z,z').
\label{eq:ot-primal}
\end{equation}
Here, $\Pi(\mu_t,\mu_v)$ denotes the set of couplings with marginals $\mu_t$ and $\mu_v$. 
To enable efficient computation, we consider the entropy-regularized optimal transport problem and denote its optimal value by
\[
\mathcal{T}_\varepsilon(\mu_t,\mu_v)
=
\inf_{\pi\in\Pi(\mu_t,\mu_v)}
\int_{\mathcal{Z}\times\mathcal{Z}} c(z,z')\,d\pi(z,z')
+
\varepsilon\,\mathrm{KL}\!\left(\pi \,\middle\|\, \mu_t\otimes\mu_v\right),
\]
where $\varepsilon>0$ is a regularization parameter. The corresponding optimization problem can be solved efficiently by the iterative Sinkhorn algorithm \cite{cuturi2013sinkhorn}.

Moreover, under mild assumptions, Kantorovich duality implies that the optimal value of the primal OT problem in \eqref{eq:ot-primal} admits the following dual representation \cite[Theorem 5.10]{villani2009optimal}:

\begin{equation*}
    \mathcal{T}(\mu_t,\mu_v)
    =
    \sup_{\substack{f\in L^1(\mu_t),\, g\in L^1(\mu_v)\\ f(z)+g(z')\le c(z,z')\ \forall z,z'\in\mathcal{Z}}}
    \left\{
    \int_{\mathcal{Z}} f(z)\,d\mu_t(z)+\int_{\mathcal{Z}} g(z)\,d\mu_v(z)
    \right\}.
	\label{eq:ot-dual}
\end{equation*}
Here, $f$ and $g$ are dual potentials, and $L^1(\mu)$ denotes the space of functions integrable with respect to a measure $\mu$.
% In the present setting, where $\mu_t$ and $\mu_v$ are empirical measures and the ground cost $c$ is given by the squared Euclidean distance, the above dual representation holds and the dual optimum is attained. 
For the empirical measures $\mu_t$ and $\mu_v$ considered here, with Euclidean cost, the above dual representation holds and the dual optimum is attained.
Since both $\mu_t$ and $\mu_v$ are probability measures, the dual potentials are unique only up to an additive constant: if $(f,g)$ is optimal, then $(f+a,g-a)$ is also optimal for any $a\in\mathbb{R}$. This non-uniqueness up to an additive constant does not affect the resulting sample importance weights, since the reweighting step is invariant to adding a constant to $f$. For notational convenience, we fix one optimal dual pair and denote it by $(f^*,g^*)$.

\paragraph{Dual Potentials as Sensitivity Signals.}

The optimal dual potentials admit a sensitivity interpretation with respect to perturbations of the probability mass. Specifically, for empirical measures with supports $\{z_i\}_{i=1}^n$ and $\{z_j'\}_{j=1}^m$, we write
\[
f_i^* = f^*(z_i) \quad \text{and} \quad g_j^* = g^*(z_j')
\]
for the values of the optimal dual potentials on the support points of $\mu_t$ and $\mu_v$, respectively. For empirical measures with fixed support, the Kantorovich dual objective is linear in the support masses for a fixed dual pair, and the optimal dual potentials appear as the corresponding coefficients. In particular, a positive value of $f_i^*$ corresponds to an increase in the transport value $\mathcal{T}(\mu_t,\mu_v)$ when the mass at $z_i$ is increased, whereas a negative value corresponds to a decrease in the transport value for the same perturbation. Accordingly, $f_i^*$ can be interpreted as a sample-level sensitivity signal that captures how $z_i$ contributes to the discrepancy between $\mu_t$ and $\mu_v$. Prior work has leveraged this interpretation for training data valuation \cite{just2023lava, li2024data, kessler2024sava}; we show that it can be developed into a principled reweighting framework for spurious correlation mitigation under label noise. 

\section{Proposed Method}

We propose POTER, which constructs a reweighted empirical training distribution by deriving sample importance weights from OT dual potentials. The key insight is that, by solving OT between the empirical training distribution and an explicitly constructed reference distribution, the training-side dual potential $f_i^{*}$ provides a principled measure of how well each training sample aligns with the desired group composition encoded by the reference distribution. We derive sample importance from these dual potentials through three steps: (i) constructing a reference distribution from limited group supervision, (ii) solving OT in a class-conditioned manner to obtain dual potentials that are not confounded by spurious attribute similarity, and (iii) converting the resulting sensitivity signals into sample weights suitable for robust reweighted ERM training. The following subsections describe each step in turn.

\subsection{Reference Distribution Construction}

We define the reference distribution $\mu_{\mathrm{ref}}$ as a mixture of group-conditional empirical validation distributions with specified mixture weights $\rho$. For each group $g \in \mathcal{G}$, let
\[
D_v^g := \{(x_j', y_j') \in D_v : g_j' = g\}
\]
denote the corresponding validation subset. Given mixture weights $\rho = (\rho_g)_{g \in \mathcal{G}}$ with $\rho_g \ge 0$ and $\sum_{g \in \mathcal{G}} \rho_g = 1$, we define
\[
\mu_{\mathrm{ref}}
:=
\sum_{g \in \mathcal{G}}
\rho_g
\left(
\frac{1}{|D_v^g|}
\sum_{(x_j', y_j') \in D_v^g}
\delta_{(x_j', y_j')}
\right).
\]

Since bias-conflicting minority groups are typically underrepresented in the training distribution but crucial for achieving worst-group robustness, $\rho$ assigns larger mass to these groups and smaller, potentially zero, mass to bias-aligned majority groups.

\subsection{Class-Conditioned Optimal Transport}

\paragraph{Dual Potentials under the Reference Distribution.}
Consider the optimal dual pair $(f^*,g^*)$ associated with the OT problem between $\mu_t$ and $\mu_{\mathrm{ref}}$. As discussed above, the quantity $f_i^*$ can be viewed as a sample-level sensitivity of the OT value with respect to perturbations of the mass at the training sample $z_i$. Samples that are overrepresented or distributionally misaligned with $\mu_{\mathrm{ref}}$ are expected to induce larger values of $f_i^*$, since increasing their mass makes the training distribution harder to transport to the reference distribution. In contrast, samples that align well with $\mu_{\mathrm{ref}}$ are expected to have smaller values of $f_i^*$. Because $\mu_{\mathrm{ref}}$ is constructed to encode the target group composition for worst-group generalization, we use $f_i^*$ as a sample-importance signal, with smaller dual potentials indicating closer alignment with the desired distribution.

\paragraph{Class-Conditioned Cost.}
However, when an attribute is spuriously correlated with the class label, bias-aligned majority samples from one class may lie close in feature space to bias-conflicting minority reference samples from another class that share the same attribute. For example, in the Waterbirds dataset, landbirds on water backgrounds and waterbirds on water backgrounds share the same background, and may therefore be close in feature space despite belonging to different classes. Solving OT jointly over all samples can therefore induce cross-class transport between such samples, causing majority samples that would otherwise be expected to receive large dual potentials to instead be matched to minority reference samples and assigned inappropriately small values. To prevent this effect, we solve OT in a class-conditioned manner using the cost
\[
C_{ij} =
\begin{cases}
	\|\phi(x_i)-\phi(x_j')\|_2^2, & \tilde y_i = y_j',\\
	+\infty, & \tilde y_i \neq y_j',
\end{cases}
\]
where $\phi(\cdot)$ denotes the (pretrained) feature embedding used to define the
transport cost.
The resulting OT problem admits an exact classwise decomposition, which we leverage in practice by solving OT separately within each class and concatenating the resulting dual potentials. The corresponding formal statement and proof are deferred to Appendix~\ref{app:class_conditioned_ot}. Consequently, the resulting dual potentials $\{f_i^*\}_{i=1}^n$ provide a sample-importance signal that reflects alignment with the class-conditional reference distribution and is not confounded by spurious cross-class similarity arising from shared spurious attributes.

\subsection{Sample Reweighting from Dual Potentials}
Having obtained the class-conditioned dual potentials $\{f_i^*\}_{i=1}^n$, we now convert them into sample importance weights to construct a reweighted empirical training distribution for robust ERM training.

To this end, let
\[
\mathcal P_n
:=
\left\{
\mu=\sum_{i=1}^n w_i\,\delta_{z_i}
\;:\;
w_i\ge 0,\ \sum_{i=1}^n w_i=1
\right\},
\]
denote the set of probability measures supported on the training samples $\{z_i\}_{i=1}^n$, and let
$
\mu_0=\sum_{i=1}^n w_i^{0}\,\delta_{z_i}
$
be a base distribution on the same support. For \(\tau>0\), we consider the KL-regularized optimization problem
\begin{equation}
\mu^*
\in
\argmin_{\mu\in\mathcal P_n}
\left\{
\sum_{i=1}^n f_i^*\,w_i
+
\tau\,\mathrm{KL}(\mu \| \mu_0)
\right\}.
\label{eq:kl_reweighting}
\end{equation}

Here, smaller values of \(f_i^*\) indicate training samples that are better aligned with the reference distribution \(\mu_{\mathrm{ref}}\). Accordingly, minimizing the linear term in \eqref{eq:kl_reweighting} assigns larger mass to such samples, thereby reversing the ordering of the dual-potential scores into an importance weighting for reweighted ERM. The KL regularization prevents the degenerate solution that would otherwise concentrate all mass on the lowest-score samples, while also yielding a more general framework that allows prior information on the training distribution to be incorporated through the base measure \(\mu_0\).
\begin{proposition}[KL-regularized reweighting]
	Assume that \(w_i^{0}>0\) for all \(i\in[n]\). Then the optimization problem in \eqref{eq:kl_reweighting} admits the unique solution
	\[
	w_i^*
	=
	\frac{w_i^{0}\exp(-f_i^*/\tau)}
	{\sum_{j=1}^n w_j^{0}\exp(-f_j^*/\tau)},
	\qquad i\in[n].
	\]
	\label{prop:kl_reweighting}
\end{proposition}
Proposition \ref{prop:kl_reweighting} yields a unique closed-form solution, giving the exponential reweighting rule. The parameter \(\tau\) controls the sharpness of the reweighting: smaller values place more mass on samples with lower dual-potential scores, whereas larger values produce a flatter weighting closer to the base distribution \(\mu_0\). 
When additional prior information on the training distribution is available---such as known group information---it can be incorporated through the choice of $\mu_0$. In our setting, however, we assume no such prior information and therefore take $\mu_0 = \mu_t$, the empirical training distribution. For the standard empirical training distribution, where \(w_i^{0}=1/n\), this reduces to
\[
w_i^*
=
\frac{\exp(-f_i^*/\tau)}
{\sum_{j=1}^n \exp(-f_j^*/\tau)},
\qquad i\in[n].
\]
Finally, letting \(h_\theta\) denote the model parameterized by \(\theta\) and \(\ell\) a standard classification loss, we train the model by minimizing the reweighted empirical risk
\[
\inf_{\theta\in\Theta}
\sum_{i=1}^n w_i^*\,\ell\bigl(h_\theta(x_i),\tilde y_i\bigr).
\]
The proof of Proposition \ref{prop:kl_reweighting} is deferred to Appendix~\ref{app:kl_reweighting}.

\section{Experiments}

\subsection{Datasets}
We conduct experiments on four widely used benchmarks for spurious correlation: CMNIST, Waterbirds, CelebA, and CivilComments.

\textbf{CMNIST} \cite{arjovsky2019invariant} is a colored variant of MNIST in which the digit label is the target and the color serves as the spurious attribute. Following prior work, digits 0--4 are assigned to label 0 and digits 5--9 to label 1, yielding four groups defined by label and color.

\textbf{Waterbirds} \cite{sagawa2019distributionally} is constructed from CUB bird images \cite{wah2011caltech} and Places backgrounds \cite{zhou2017places}, with groups defined by bird type and background. The target label is bird type, while background is spuriously correlated with the label in the training set. The minority groups are \textit{landbird on water background} and \textit{waterbird on land background}.

\textbf{CelebA} \cite{liu2015deep} is used for blond-hair classification, where gender serves as the spurious attribute. Groups are defined by the combination of hair color and gender, and the minority group is \textit{blond male}.

\textbf{CivilComments} \cite{borkan2019nuanced} is a text toxicity
classification benchmark. We use the WILDS version of the dataset \cite{koh2021wilds}, in which identity-related text can act as spurious attributes. Following prior work \cite{liu2021just,zhang2022correct,hanimproving,gourabathina2026robustness}, we consider eight identity attributes: \textit{male}, \textit{female}, \textit{LGBTQ}, \textit{Christian}, \textit{Muslim}, \textit{other religion}, \textit{black}, and \textit{white}, defining 16 subgroups.

Further details on dataset construction and experimental settings are provided in \autoref{sec:appendix-dataset}.

\subsection{Experimental Settings}

\begin{table*}[t]
	\centering
	\caption{Worst-group and average accuracy on CMNIST, Waterbirds, CelebA, and CivilComments under the standard evaluation setting. \emph{Group Labels} indicates whether the method uses subgroup annotations from the training data, the validation data, or both.
    \emph{No Extra Training} indicates whether the method operates without additional retraining. Among methods that use only validation group labels, boldface indicates the best performance and underlining denotes the second-best.}
	\label{tab:main_results}
	\small
	\setlength{\tabcolsep}{5pt}
	\renewcommand{\arraystretch}{1.15}
	\resizebox{\textwidth}{!}{%
	\begin{tabular}{l c c cc cc cc cc}
		\specialrule{1.2pt}{0pt}{0pt}
		\multirow{2}{*}[-1.0ex]{Algorithm} &
		\multirow{2}{*}[-1.0ex]{\makecell[c]{Group Labels\\(Tr/Val)}} &
		\multirow{2}{*}[-1.0ex]{\makecell[c]{No Extra\\Training}} &
		\multicolumn{2}{c}{\raisebox{-0.6ex}{CMNIST}} &
		\multicolumn{2}{c}{\raisebox{-0.6ex}{Waterbirds}} &
		\multicolumn{2}{c}{\raisebox{-0.6ex}{CelebA}} &
		\multicolumn{2}{c}{\raisebox{-0.6ex}{CivilComments}} \\
		\cmidrule(lr){4-5} \cmidrule(lr){6-7} \cmidrule(lr){8-9} \cmidrule(lr){10-11}
		& & & Worst Acc & Avg Acc & Worst Acc & Avg Acc & Worst Acc & Avg Acc & Worst Acc & Avg Acc \\
		\midrule
		
		Group DRO             & Tr/Val & \checkmark & 73.1$_{\pm 0.3}$ & 74.8$_{\pm 0.2}$ & 90.6$_{\pm 0.2}$ & 92.7$_{\pm 0.1}$ & 89.3$_{\pm 1.3}$ & 92.6$_{\pm 0.3}$ & 69.0$_{\pm 0.9}$ & 89.9$_{\pm 0.2}$ \\
		LISA                 & Tr/Val & \checkmark & 73.3$_{\pm 0.2}$ & 74.0$_{\pm 0.1}$ & 89.2$_{\pm 0.6}$ & 91.8$_{\pm 0.3}$ & 89.3$_{\pm 1.1}$ & 92.4$_{\pm 0.4}$ & 72.6$_{\pm 0.1}$ & 89.2$_{\pm 0.9}$ \\
		DFR$^{\mathrm{Tr}}$  & Tr/Val & -          & 59.8$_{\pm 0.4}$ & 62.1$_{\pm 0.2}$ & 90.2$_{\pm 0.8}$ & 97.0$_{\pm 0.3}$ & 80.7$_{\pm 2.4}$ & 90.6$_{\pm 0.7}$ & 58.0$_{\pm 1.3}$ & 92.0$_{\pm 0.1}$ \\
		PDE                  & Tr/Val & \checkmark & 72.6$_{\pm 0.7}$ & 73.0$_{\pm 0.4}$ & 90.3$_{\pm 0.3}$ & 92.4$_{\pm 0.8}$ & 91.0$_{\pm 0.4}$ & 92.0$_{\pm 0.6}$ & 71.5$_{\pm 0.5}$ & 86.3$_{\pm 1.7}$ \\
		
		\specialrule{0.8pt}{2pt}{0pt}
		\noalign{\vskip 5pt}
		\specialrule{0.8pt}{0pt}{2pt}
		
		ERM                  & -     & \checkmark & 3.4$_{\pm 0.9}$  & 12.9$_{\pm 0.8}$ & 62.6$_{\pm 0.3}$ & 97.3$_{\pm 1.0}$ & 47.7$_{\pm 2.1}$ & 94.9$_{\pm 0.3}$ & 58.6$_{\pm 1.7}$ & 92.1$_{\pm 0.4}$ \\
		JTT                  & Val   & -          & 67.3$_{\pm 5.1}$ & 76.4$_{\pm 3.3}$ & 83.8$_{\pm 1.2}$ & 89.3$_{\pm 0.7}$ & 81.5$_{\pm 1.7}$ & 88.1$_{\pm 0.3}$ & 69.3$_{\mathrm{N/A}}$ & 91.1$_{\mathrm{N/A}}$ \\
		CnC                  & Val   & -          & --                & --                & 88.5$_{\pm 0.3}$ & 90.9$_{\pm 0.1}$ & 88.8$_{\pm 0.9}$ & 89.9$_{\pm 0.5}$ & 68.9$_{\pm 2.1}$ & 81.7$_{\pm 0.5}$ \\
        SSA                  & Val   & -          & 71.1$_{\pm 0.4}$ & 75.0$_{\pm 0.3}$ & 89.0$_{\pm 0.6}$ & 92.2$_{\pm 0.9}$ & \underline{89.8}$_{\pm 1.3}$ & 92.8$_{\pm 0.1}$ & 69.9$_{\pm 2.0}$ & 88.2$_{\pm 2.0}$ \\

		AFR                  & Val   & -          &  53.5$_{\pm 0.0}$              &  60.6$_{\pm 0.0}$             & 90.4$_{\pm 1.1}$ & 94.2$_{\pm 1.2}$ & 82.0$_{\pm 0.5}$ & 91.3$_{\pm 0.3}$ & 68.7$_{\pm 0.6}$ & 89.8$_{\pm 0.6}$ \\
		GIC                  & Val   & -          & \underline{72.2}$_{\pm 0.5}$ & 73.2$_{\pm 0.2}$ & 86.3$_{\pm 0.1}$ & 89.6$_{\pm 1.3}$ & 89.4$_{\pm 0.2}$ & 91.9$_{\pm 0.1}$ & \underline{72.5}$_{\pm 0.3}$ & 90.0$_{\pm 0.3}$ \\
		LEIA                 & Val   & -          &  56.0$_{\pm 4.8}$              &  63.2$_{\pm 2.0}$        & \underline{90.7}$_{\pm 0.2}$ & 93.3$_{\pm 0.7}$ & 85.0$_{\pm 0.9}$ & 95.2$_{\pm 0.1}$ & \textbf{72.9}$_{\pm 0.2}$ & 90.9$_{\pm 0.8}$ \\
		
		\specialrule{0.8pt}{0pt}{2pt}
        \rowcolor{yellow!12}
		POTER (Ours)         & Val   & \checkmark &  \textbf{72.5}$_{\pm 0.7}$ & 73.6$_{\pm 0.4}$               & \textbf{90.9}$_{\pm 0.3}$ & 92.7$_{\pm 0.1}$ & \textbf{90.4}$_{\pm 0.5}$ & 91.2$_{\pm 0.3}$ & 71.0$_{\pm 0.6}$ & 89.9$_{\pm 0.3}$ \\
		\specialrule{1.2pt}{0pt}{0pt}
	\end{tabular}
}
\end{table*}

We consider three experimental settings: the standard benchmark setting, the symmetric label noise setting, and a subgroup-concentrated label noise setting that we introduce.

\paragraph{Standard Benchmark Setting.}
The standard benchmark setting evaluates worst-group robustness under spurious correlations using the original benchmarks without additional label corruption. We note that label noise is already present in
standard benchmarks: CMNIST contains 25\% label noise by construction, and the
Waterbirds benchmark contains mislabeled examples \cite{asgari2022masktune},
suggesting that label noise is a naturally occurring phenomenon in spurious
correlation benchmarks.

\paragraph{Symmetric Label Noise.}
Following prior work \cite{NEURIPS2024_fba9133d}, we inject symmetric label noise into the training set by independently flipping each training label with probability $\eta$. We vary $\eta$ across a range of noise levels to evaluate robustness under different degrees of label corruption.

% \paragraph{Subgroup-Concentrated Label Noise.}
% We introduce an additional noise setting in which corruption is concentrated within minority subgroups rather than distributed uniformly across the training set. While symmetric label noise corrupts all groups at the same rate—up to 30\% in our experiments—subgroup-concentrated noise affects only minority subgroups at rates up to 60\%, leaving the remaining groups uncorrupted. Because minority subgroups are small, the overall training-set noise rate remains low, yet the corruption within the affected subgroups is severe. This setting probes whether methods remain robust to locally concentrated label noise, which poses a distinct challenge from globally distributed symmetric noise.

\paragraph{Subgroup-Concentrated Label Noise.}
We introduce a noise setting in which corruption is concentrated within minority
subgroups at rates up to 60\%, leaving the remaining groups uncorrupted. While
symmetric label noise corrupts all groups uniformly, this setting probes whether
methods remain robust to locally concentrated label noise—a distinct and
previously overlooked challenge.

\subsection{Baselines}

We compare our method against representative baselines under two settings: one in which training group labels are available and one in which only validation group labels are available.
For methods with training group labels, we consider Group DRO \cite{sagawa2019distributionally}, LISA \cite{yao2022improving}, DFR \cite{kirichenkolast}, and PDE \cite{deng2024robust}. For methods with only validation group labels, we consider JTT \cite{liu2021just}, CnC \cite{zhang2022correct}, SSA \cite{nam2022spread}, AFR \cite{qiu2023simple}, GIC \cite{hanimproving}, and LEIA \cite{gourabathina2026robustness}. 
In the label-noise setting, we additionally focus on Group DRO, a representative loss-based robust optimization method;
AFR, which performs loss-based sample upweighting for last-layer retraining; 
and KNN-RAD \cite{NEURIPS2024_fba9133d}, the most closely related method that directly addresses label noise under spurious correlations. Detailed descriptions are provided in \autoref{sec:appendix-baseline}.

\subsection{Results}

\begin{figure}[t]
    \centering

    \subfigure[Reference set (left) and noisy training set (right).]{
        \begin{minipage}[t]{0.48\textwidth}
            \centering
            \includegraphics[height=3.5cm]{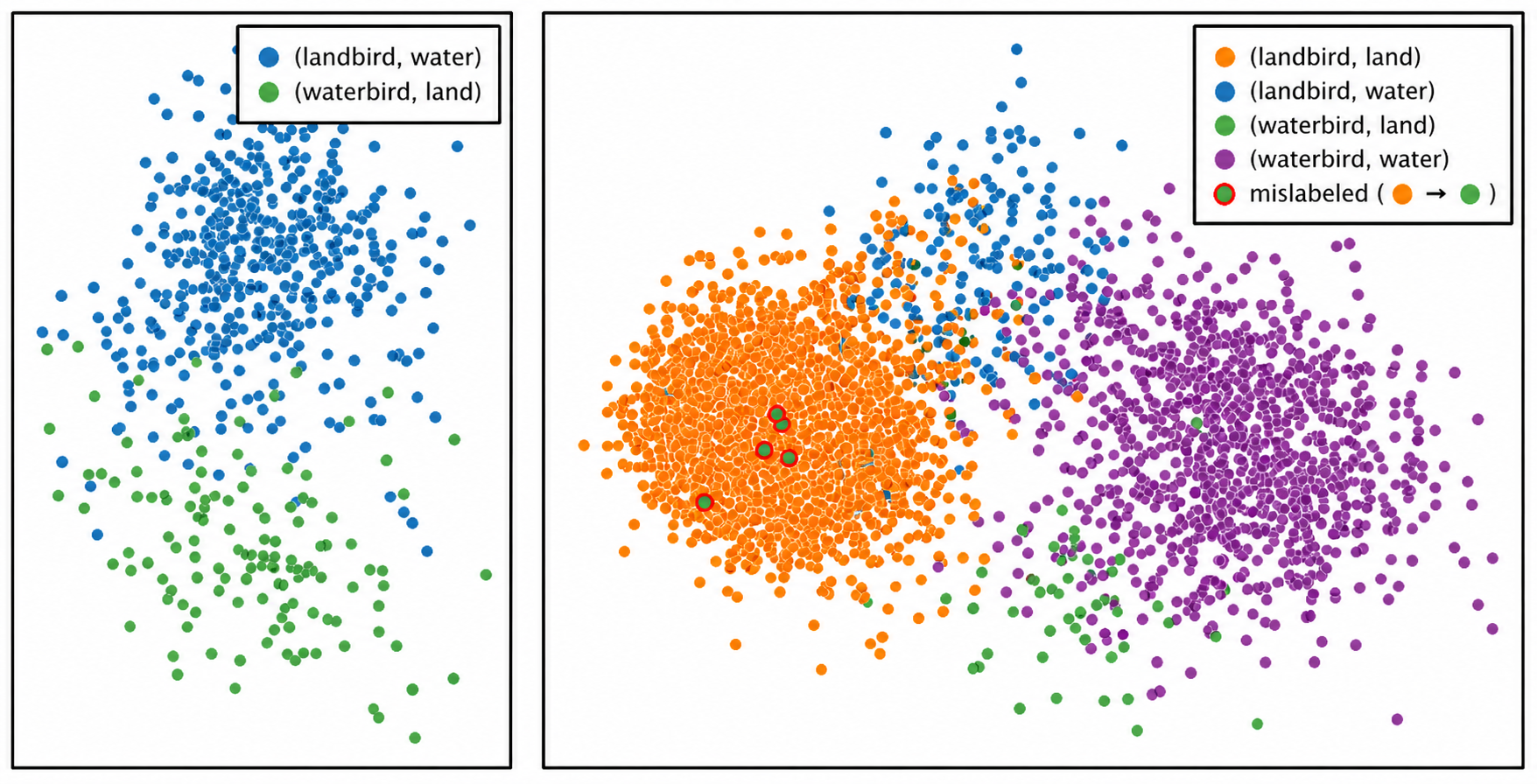}

            \label{fig:waterbird_train_ref_pca}
        \end{minipage}
    }
    \hfill
    \subfigure[$f^*$ values for training samples labeled ``waterbird''.]{
        \begin{minipage}[t]{0.48\textwidth}
            \centering
            \includegraphics[height=3.5cm]{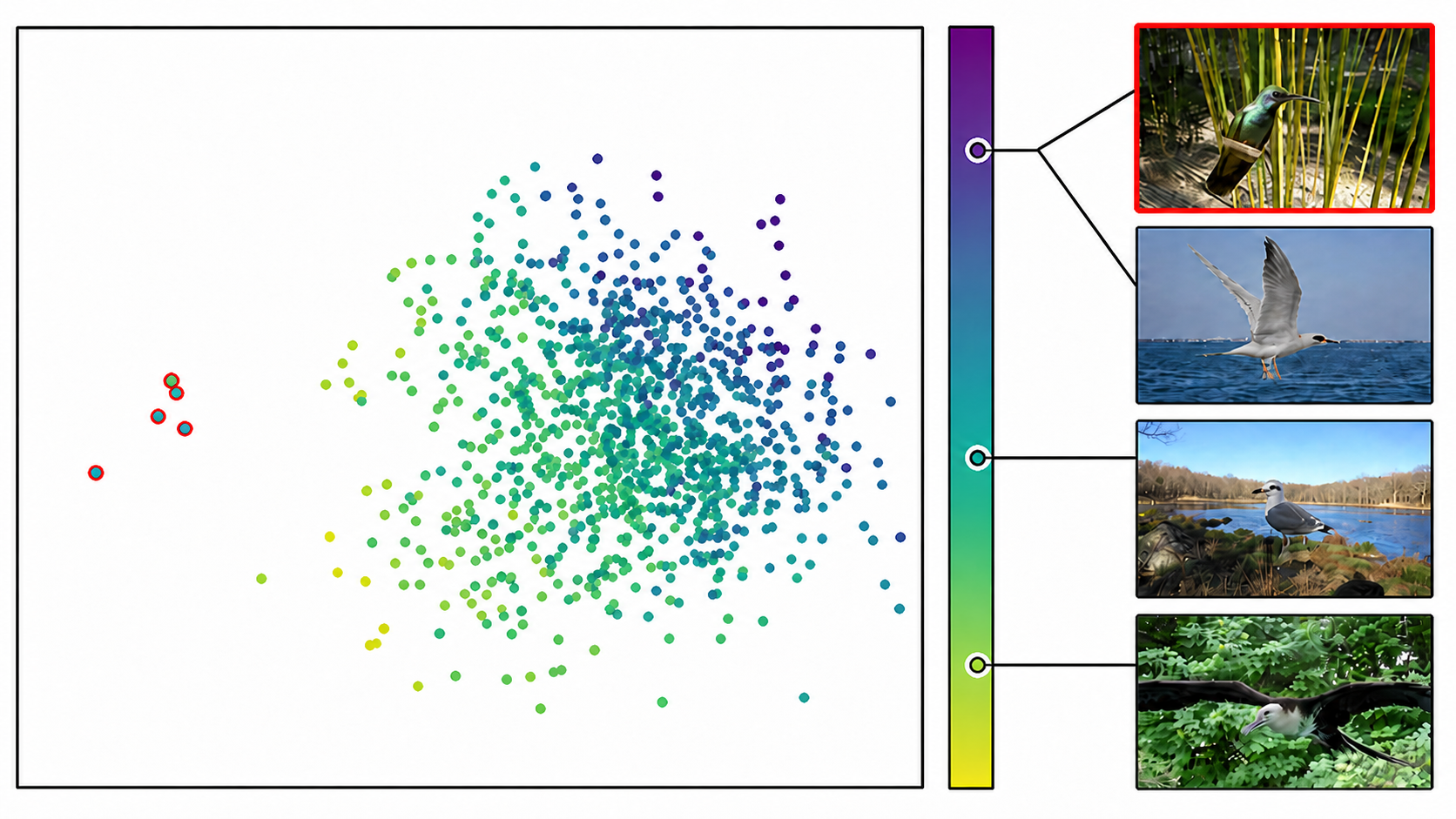}
            \label{fig:waterbird_classwise_fstar}
        \end{minipage}
    }

    \caption{
    Analysis of reference-based sample importance on Waterbirds.
    (a) PCA projection of feature embeddings for the reference set and noisy training set.
    The reference set consists of validation samples from the two bias-conflicting minority groups, while the noisy training set contains all four groups with mislabeled samples marked by red outlines.
    (b) Training samples with observed label ``waterbird'' are colored by the class-conditioned OT dual potential $f^*$ computed against the waterbird-on-land reference samples.
    Brighter colors indicate smaller $f^*$ values and therefore larger sample importance.
    The visual examples, from top to bottom, show a mislabeled sample, a strongly bias-aligned sample, a bias-aligned sample containing both land and water background regions, and a bias-conflicting sample.
    }
    \label{fig:waterbird_analysis}
\end{figure}

\subsubsection{Standard Spurious Correlation Benchmarks}

Table~\ref{tab:main_results} reports results on standard spurious correlation
benchmarks. Among methods that use only validation group labels, POTER achieves
superior worst-group accuracy on CMNIST, Waterbirds, and CelebA, and remains
competitive on CivilComments, without relying on additional training stages.
Notably, this performance matches or exceeds several methods with access to full training group labels.

Figure~\ref{fig:waterbird_analysis} offers insight into the mechanism underlying
POTER's strong performance. As shown in Figure~\ref{fig:waterbird_analysis}(b),
bias-conflicting samples receive high importance, whereas strongly bias-aligned
samples receive low importance. Beyond this group-level distinction, samples
near the boundary between land and water backgrounds retain relatively high
importance even within the bias-aligned group. Unlike group-level reweighting
approaches such as Group DRO, POTER assigns fine-grained weights according to
alignment with the reference distribution, allowing informative samples to be
emphasized regardless of their group membership. This sample-level flexibility
is consistent with the strong worst-group accuracy observed in
Table~\ref{tab:main_results}.

\subsubsection{Robustness under Label Noise}

\paragraph{Interpretation of Robustness under Label Noise.}
The label-noise robustness of POTER follows from using class-conditioned
reference alignment rather than loss-based sample selection. For samples with
observed label ``waterbird'', POTER compares them to the waterbird-on-land
reference distribution. As illustrated by the red-outlined mislabeled samples
in Figure~\ref{fig:waterbird_analysis}, a landbird-on-land sample mislabeled
as ``waterbird'' is evaluated within the observed waterbird class, but its
feature embedding remains far from the waterbird-on-land reference. It therefore
receives a large $f^*$ value and hence low importance, as do strongly
bias-aligned samples. Thus, the alignment-based weighting rule effectively
downweights both mislabeled and strongly bias-aligned samples through a single
unified mechanism.

\begin{table*}[t]
	\centering
    \caption{Worst-group accuracy on Waterbirds and CelebA under different symmetric label noise rates.}
	\label{tab:waterbirds_celeba_label_noise_results}
    \footnotesize
	\setlength{\tabcolsep}{6pt}
    \renewcommand{\arraystretch}{0.88}
	\resizebox{\textwidth}{!}{%
	\begin{tabular}{l l c c cccc}
		\specialrule{1.2pt}{0pt}{0pt}
		\multirow{2}{*}[-1.0ex]{Dataset} &
		\multirow{2}{*}[-1.0ex]{Method} &
		\multirow{2}{*}[-1.0ex]{\makecell[c]{Group Labels\\(Tr/Val)}} &
		\multirow{2}{*}[-1.0ex]{\makecell[c]{No Extra\\Training}} &
		\multicolumn{4}{c}{\raisebox{-0.6ex}{Label Noise (\%)}} \\
		\cmidrule(lr){5-8}
		& & & & 0 & 10 & 20 & 30 \\
		\midrule
		
		\multirow{4}{*}{Waterbirds}
        & Group DRO
        & Tr/Val & \checkmark
        & 90.6$_{\pm 0.2}$
        & 72.9$_{\pm 1.6}$
        & 54.3$_{\pm 1.0}$
        & 52.2$_{\pm 3.6}$ \\

        & AFR
        & Val & -
        & 88.3$_{\pm 0.9}$
        & 58.7$_{\pm 0.4}$
        & 61.2$_{\pm 5.9}$
        & 52.9$_{\pm 0.0}$ \\

		\cmidrule(lr){2-8}

		& KNN-RAD
		& Val & -
		& \textbf{91.0}$_{\pm 0.1}$
		& 82.4$_{\pm 0.7}$
		& 74.7$_{\pm 1.1}$
		& 68.9$_{\pm 2.1}$\\
		
		& POTER (Ours)
		& Val & \checkmark
		& \cellcolor{yellow!12} 90.9$_{\pm 0.3}$ 
        & \cellcolor{yellow!12} \textbf{89.4}$_{\pm 0.4}$  
        & \cellcolor{yellow!12} \textbf{87.9}$_{\pm 1.6}$  
        & \cellcolor{yellow!12} \textbf{85.3}$_{\pm 1.4}$   \\

		\specialrule{0.8pt}{2pt}{2pt}

		\multirow{4}{*}{CelebA}
		& Group DRO
		& Tr/Val & \checkmark
		& 89.3$_{\pm 1.3}$
		& 66.9$_{\pm 0.3}$
		& 59.5$_{\pm 2.6}$
		& 54.8$_{\pm 1.8}$ \\
        
        & AFR
        & Val & -
        & 80.9$_{\pm 2.3}$
        & 69.4$_{\pm 4.8}$
        & 67.7$_{\pm 2.4}$
        & 58.8$_{\pm 6.3}$ \\

		\cmidrule(lr){2-8}

		& KNN-RAD
		& Val & -
		& 83.9$_{\pm 0.0}$
		& 80.0$_{\pm 0.9}$
		& 78.1$_{\pm 1.3}$
		& 70.2$_{\pm 2.6}$\\
      
		& POTER (Ours)
		& Val & \checkmark
        
		& \cellcolor{yellow!12} \textbf{90.4}$_{\pm 0.5}$ 
		& \cellcolor{yellow!12} \textbf{87.6}$_{\pm 0.3}$ 
		& \cellcolor{yellow!12} \textbf{82.8}$_{\pm 1.0}$ 
		& \cellcolor{yellow!12} \textbf{78.2}$_{\pm 0.4}$ \\
		
		\specialrule{1.2pt}{0pt}{0pt}
	\end{tabular}
}
\end{table*}

\begin{figure}[t]
	\centering
	\subfigure[Waterbirds.]{
		\includegraphics[width=0.48\textwidth]{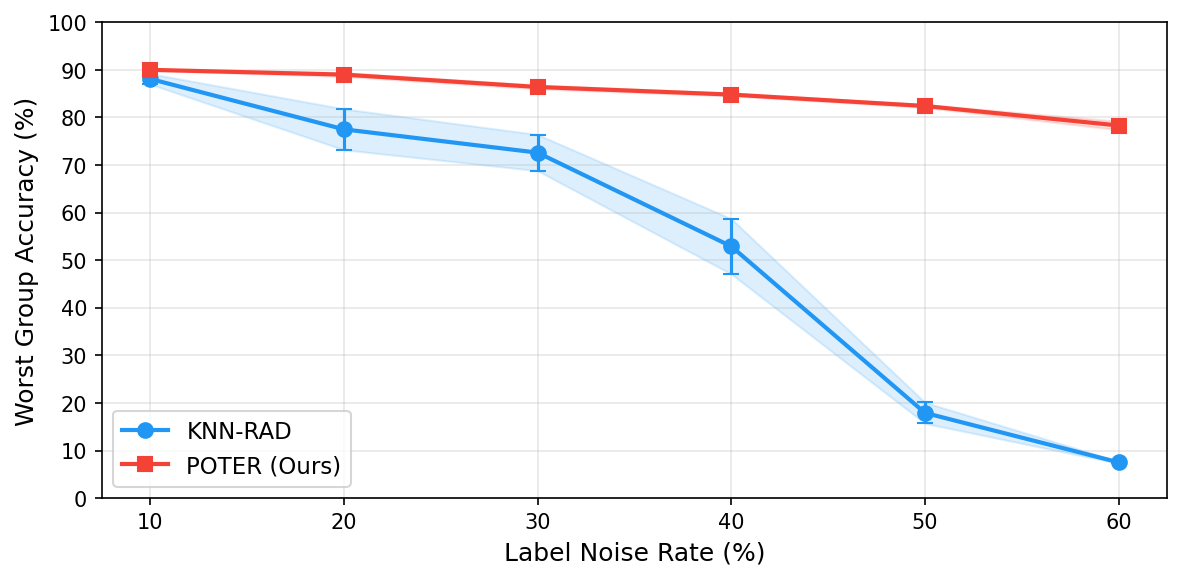}
		\label{fig:waterbirds_subgroup_concentrated_label_noise}
	}
	\hfill
	\subfigure[CelebA.]{
		\includegraphics[width=0.48\textwidth]{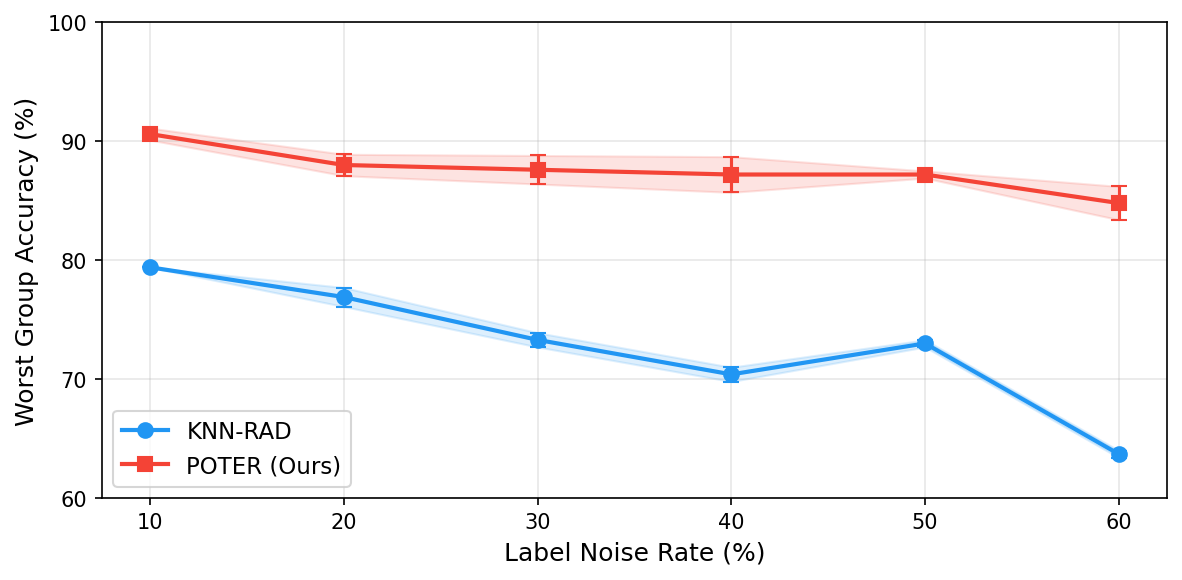}
		\label{fig:celeba_subgroup_concentrated_label_noise}
	}
\caption{Worst-group accuracy under subgroup-concentrated label noise on Waterbirds and CelebA. Label noise is concentrated on the minority subgroups: $(\text{waterbird}, \text{land background})$ and $(\text{landbird}, \text{water background})$ for Waterbirds, and $(\text{blond}, \text{male})$ for CelebA. The y-axis range is adjusted for each dataset to highlight within-dataset performance differences.}
	\label{fig:subgroup_concentrated_label_noise}
\end{figure}

\paragraph{Symmetric Label Noise.}
Table~\ref{tab:waterbirds_celeba_label_noise_results} reports worst-group
accuracy on Waterbirds and CelebA under increasing symmetric label noise.
Loss-based methods degrade substantially as noise increases: on Waterbirds,
Group DRO drops from 90.6\% to 52.2\% and AFR from 88.3\% to 52.9\% at 30\%
noise, consistent with the fragility of loss-based signals under label
corruption. Since methods such as SSA and GIC incorporate Group DRO as part of
their training pipeline, the degradation of Group DRO is indicative of a
vulnerability shared by this broader class of approaches. KNN-RAD, which
explicitly addresses label noise, is more stable than loss-based baselines but
is consistently outperformed by POTER once label noise is introduced. At 30\%
noise, POTER achieves 85.3\% on Waterbirds and 78.2\% on CelebA, outperforming
KNN-RAD by 16.4 and 8.0 percentage points, respectively.

\paragraph{Subgroup-Concentrated Label Noise.}
Figure~\ref{fig:subgroup_concentrated_label_noise} reports worst-group accuracy
under subgroup-concentrated label noise, where corruption is restricted to
minority subgroups. KNN-RAD degrades sharply over the evaluated range of
subgroup noise rates, falling from 88.1\% to 7.5\% on Waterbirds and from
79.4\% to 63.7\% on CelebA between 10\% and 60\% subgroup noise. This suggests
that local-neighborhood-based label correction can become unreliable when label
corruption is concentrated within the subgroups most critical for worst-group
generalization. In contrast, POTER remains stable across the same range,
achieving 78.3\% on Waterbirds and 84.8\% on CelebA at 60\% noise,
demonstrating robustness to locally concentrated label corruption.

\subsubsection{Analysis of the Weighting Signal}
We examine the weighting signal before reweighted ERM.
Figure~\ref{fig:fstar_group_boxplot} shows the group-wise distributions of
the class-conditioned OT dual potentials $f^*$ on Waterbirds and CelebA, and
Table~\ref{tab:group_mean_weights} reports the resulting mean sample weight
per group on Waterbirds, CelebA, and CMNIST, normalized so that uniform
weighting equals~1 (group definitions follow
Appendix~\ref{sec:appendix-dataset}).

Two properties stand out. First, within each class, the bias-conflicting
groups receive lower $f^*$ than their bias-aligned counterparts, and hence
larger weights: 6.32 and 10.19 for the two Waterbirds minority groups, and
2.22 and 1.91 for the two bias-conflicting CMNIST groups. In CelebA, the
within-class mean weights are 0.93 for $g_1$ versus 0.35 for $g_2$, and
12.68 for $g_4$ versus 2.41 for $g_3$. Class balancing assigns equal total
mass to each class, giving the much larger non-blond class a lower
per-sample weight scale. Thus, $g_1$ remains relatively upweighted within
its class despite its mean weight being below one.

Second, the weights vary substantially within groups rather than
reducing to group-level coefficients. On Waterbirds, the interquartile range
of the weights spans 0.34--7.79 for landbirds on water and 0.91--16.07 for
waterbirds on land. As illustrated in Figure~\ref{fig:waterbird_analysis}(b),
this variation is informative even inside bias-aligned groups: a
waterbird-on-water image with a mixed land--water background receives a
larger weight than a prototypical open-water image, consistent with weaker
reliance on the background shortcut. POTER thus emphasizes underrepresented
groups while retaining sample-level variation within them.

\begin{figure}[t]
	\centering
	\includegraphics[width=\textwidth]{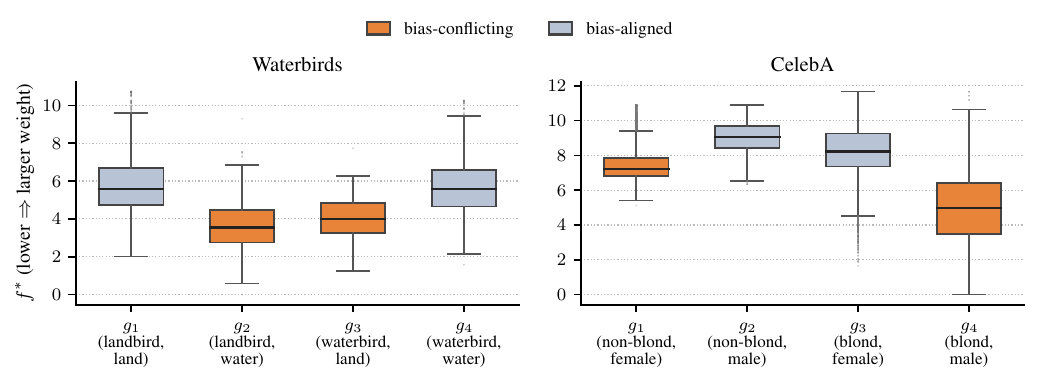}
	\caption{Distribution of class-standardized OT dual potentials $f^*$ by group on
	Waterbirds and CelebA in the standard setting. Within each class, lower
	potentials favor larger sample weights. Bias-conflicting groups (orange)
	have lower median potentials than their bias-aligned counterparts (light blue).}
	\label{fig:fstar_group_boxplot}

	\vspace{1.2em}
	% Table 3 lives inside this float so it always follows Figure 3.
	\captionof{table}{Mean sample weight per group in the standard setting, normalized so that
	uniform weighting equals~1. Boldface indicates the bias-conflicting groups.}
	\label{tab:group_mean_weights}
	\vspace{0.5em}
	\small
	\setlength{\tabcolsep}{10pt}
	\renewcommand{\arraystretch}{1.2}
	\begin{tabular}{l c c c c}
	\specialrule{1.2pt}{0pt}{0pt}
	Dataset & $g_1$ & $g_2$ & $g_3$ & $g_4$ \\
	\midrule
	Waterbirds & 0.31 & $\mathbf{6.32}$ & $\mathbf{10.19}$ & 1.87 \\
	CelebA     & $\mathbf{0.93}$ & 0.35 & 2.41 & $\mathbf{12.68}$ \\
	CMNIST     & $\mathbf{2.22}$ & 0.78 & 0.68 & $\mathbf{1.91}$ \\
	\specialrule{1.2pt}{0pt}{0pt}
	\end{tabular}
\end{figure}

\section{Conclusion}
We presented POTER, a principled reweighting framework based on optimal
transport that derives sample importance from distributional alignment. Our
experiments reveal a key vulnerability of existing methods: the premise that
high-loss samples reliably identify bias-conflicting examples can break down
under label noise, with further degradation when corruption is concentrated
within minority subgroups. While prior work has diagnosed the label-noise
sensitivity of existing mitigation methods~\cite{qiaounderstanding}, targeted
remedies remain limited: they can sacrifice clean-setting
performance~\cite{oh2022improving} or remain tied to retraining-based
pipelines~\cite{NEURIPS2024_fba9133d}. POTER addresses this issue through a
single unified weighting mechanism that handles spurious correlation and label
noise simultaneously, requiring no retraining and adding only minimal overhead
for computing OT dual potentials. We hope this work motivates broader
consideration of overlooked failure modes in robust learning under
subpopulation shift, particularly those induced by realistic data imperfections
such as label noise. One limitation is that POTER still depends on a
representative, class-labeled validation set. Appendix~\ref{sec:appendix-validation-robustness}
examines reduced and noisy references, while robustness to entirely missing
validation subgroups remains an open challenge.
\newpage

\begin{ack}
This work was supported by the National Research Foundation of Korea (NRF)
grant funded by the Korea government (MSIT; RS-2026-25468697,
RS-2026-25607210), and by the Institute of Information \& Communications
Technology Planning \& Evaluation (IITP) Global Data-X Leader HRD program
grant funded by the Korea government (MSIT; IITP-2024-RS-2024-00441244).
\end{ack}

\bibliography{references}
\bibliographystyle{plain}

%%%%%%%%%%%%%%%%%%%%%%%%%%%%%%%%%%%%%%%%%%%%%%%%%%%%%%%%%%%%
\newpage
\appendix

\section{Class-Conditioned OT Decomposition}
\label{app:class_conditioned_ot}

Let
\[
\mu_t=\sum_{i=1}^n w_i^t\,\delta_{z_i},
\qquad
\mu_{\mathrm{ref}}=\sum_{j=1}^m w_j^{\mathrm{ref}}\,\delta_{z_j'}.
\]
For each class \(k\in[K]\), define
\[
I_k:=\{i:\tilde y_i=k\},
\qquad
J_k:=\{j:y_j'=k\},
\]
and assume that the classwise masses match:
\[
\sum_{i\in I_k} w_i^t=\sum_{j\in J_k} w_j^{\mathrm{ref}},
\qquad \forall k\in[K].
\]
This condition requires equal training and reference mass within each class,
while allowing the total mass to differ across classes.

For notational simplicity, we continue to write \(x_i\) and \(x_j'\) for the feature embeddings on which the transport cost is computed. We use the class-conditioned cost
\[
C_{ij}=
\begin{cases}
	\|x_i-x_j'\|_2^2, & \tilde y_i=y_j',\\
	+\infty, & \tilde y_i\neq y_j',
\end{cases}
\]
and define the corresponding classwise empirical measures by
\[
\mu_t^{(k)} := \sum_{i\in I_k} w_i^t\,\delta_{z_i},
\qquad
\mu_{\mathrm{ref}}^{(k)} := \sum_{j\in J_k} w_j^{\mathrm{ref}}\,\delta_{z_j'}.
\]
Accordingly, \(\mathcal{T}(\mu_t,\mu_{\mathrm{ref}})\) denotes the OT value under the class-conditioned cost \(C\), and
\(\mathcal{T}^{(k)}(\mu_t^{(k)},\mu_{\mathrm{ref}}^{(k)})\) denotes its restriction to class \(k\).

Under this cost, every finite-cost coupling is supported on same-class pairs.
The mass-matching condition ensures that a feasible coupling exists within each
class, leading to the following decomposition.

\begin{proposition}[Class-conditioned OT decomposition]
	Under the class-conditioned cost \(C\), the OT value decomposes across classes as
	\[
	\mathcal{T}(\mu_t,\mu_{\mathrm{ref}})
	=
	\sum_{k=1}^K \mathcal{T}^{(k)}\bigl(\mu_t^{(k)},\mu_{\mathrm{ref}}^{(k)}\bigr).
	\]
	Moreover, if \((f^{(k)*},g^{(k)*})\) is an optimal dual pair for the \(k\)-th classwise OT problem for each \(k\in[K]\), then the global dual vectors \(f^*\in\mathbb{R}^n\) and \(g^*\in\mathbb{R}^m\), defined by
	\[
	f_i^* := f_i^{(k)*}\quad \text{for } i\in I_k,
	\qquad
	g_j^* := g_j^{(k)*}\quad \text{for } j\in J_k,
	\]
	form an optimal dual pair for \(\mathcal{T}(\mu_t,\mu_{\mathrm{ref}})\).
	\label{prop:class_conditioned_ot}
\end{proposition}

\begin{proof}
	Under the class-conditioned cost
	\[
	C_{ij}=
	\begin{cases}
		\|x_i-x_j'\|_2^2, & \tilde y_i=y_j',\\
		+\infty, & \tilde y_i\neq y_j',
	\end{cases}
	\]
	any feasible coupling \(\pi\) satisfying \(\langle C,\pi\rangle<\infty\) must satisfy
	\[
	\pi_{ij}=0
	\qquad \text{whenever } \tilde y_i\neq y_j'.
	\]
	Hence every finite-cost feasible plan is supported on
	\[
	\bigcup_{k=1}^K (I_k\times J_k).
	\]
	For each \(k\in[K]\), let \(\pi^{(k)}\) denote the restriction of \(\pi\) to \(I_k\times J_k\). Then the marginal constraints imply
	\[
	\pi^{(k)}\in \Pi\bigl(\mu_t^{(k)},\mu_{\mathrm{ref}}^{(k)}\bigr).
	\]
	Conversely, by the classwise mass-matching assumption, any collection of classwise couplings
	\[
	\pi^{(k)}\in \Pi\bigl(\mu_t^{(k)},\mu_{\mathrm{ref}}^{(k)}\bigr),
	\qquad k\in[K],
	\]
	induces a feasible global coupling \(\pi\) by setting
	\[
	\pi_{ij}:=
	\begin{cases}
		\pi^{(k)}_{ij}, & (i,j)\in I_k\times J_k \text{ for some } k,\\
		0, & \text{otherwise}.
	\end{cases}
	\]
	Therefore, the feasible set of finite-cost couplings factorizes across classes.
	
	On this feasible set,
	\[
	\langle C,\pi\rangle
	=
	\sum_{k=1}^K \langle C^{(k)},\pi^{(k)}\rangle,
	\]
	where \(C^{(k)}\) denotes the restriction of \(C\) to \(I_k\times J_k\). It follows that
	\[
	\mathcal{T}(\mu_t,\mu_{\mathrm{ref}})
	=
	\inf_{\pi\in\Pi(\mu_t,\mu_{\mathrm{ref}})} \langle C,\pi\rangle
	=
	\sum_{k=1}^K
	\inf_{\pi^{(k)}\in\Pi(\mu_t^{(k)},\mu_{\mathrm{ref}}^{(k)})}
	\langle C^{(k)},\pi^{(k)}\rangle
	=
	\sum_{k=1}^K \mathcal{T}^{(k)}\bigl(\mu_t^{(k)},\mu_{\mathrm{ref}}^{(k)}\bigr).
	\]
	
	Now let \((f^{(k)*},g^{(k)*})\) be an optimal dual pair for the \(k\)-th classwise OT problem. Define global vectors \(f^*\in\mathbb{R}^n\) and \(g^*\in\mathbb{R}^m\) by
	\[
	f_i^*:=f_i^{(k)*}\quad \text{for } i\in I_k,
	\qquad
	g_j^*:=g_j^{(k)*}\quad \text{for } j\in J_k.
	\]
	Since \(C_{ij}=+\infty\) whenever \(\tilde y_i\neq y_j'\), the constraint
	\[
	f_i^*+g_j^*\le C_{ij}
	\]
	is automatically satisfied whenever \(i\in I_k\), \(j\in J_\ell\), and \(k\neq \ell\), and on each set \(I_k\times J_k\) it reduces to the classwise dual feasibility constraint. Thus \((f^*,g^*)\) is feasible for the global dual problem. Its dual objective value is
	\[
	\sum_{i=1}^n f_i^*\,w_i^t+\sum_{j=1}^m g_j^*\,w_j^{\mathrm{ref}}
	=
	\sum_{k=1}^K
	\left(
	\sum_{i\in I_k} f_i^{(k)*}w_i^t
	+
	\sum_{j\in J_k} g_j^{(k)*}w_j^{\mathrm{ref}}
	\right),
	\]
	which, by optimality of each classwise dual pair, equals
	\[
	\sum_{k=1}^K \mathcal{T}^{(k)}\bigl(\mu_t^{(k)},\mu_{\mathrm{ref}}^{(k)}\bigr)
	=
	\mathcal{T}(\mu_t,\mu_{\mathrm{ref}}).
	\]
	Hence \((f^*,g^*)\) is an optimal dual pair for \(\mathcal{T}(\mu_t,\mu_{\mathrm{ref}})\).
\end{proof}

The decomposition is determined by the observed class partition and matched
marginal masses and therefore does not require any assumption on the
label-corruption mechanism.
In our implementation, the two empirical measures are normalized within each
class to satisfy the mass-matching condition, as described in
Appendix~\ref{sec:appendix-ot-computation}.

\paragraph{Computational Implication.}
Proposition \ref{prop:class_conditioned_ot} also yields a practical computational benefit in a classwise implementation. Let
\[
n_k := |I_k|,
\qquad
m_k := |J_k|.
\]
Rather than operating on a dense \(n\times m\) transport matrix, the class-conditioned formulation can be implemented as \(K\) independent OT subproblems of sizes \(n_k\times m_k\). For Sinkhorn-style solvers, this reduces the total per-iteration cost from \(O(nm)\) for a dense full-matrix implementation to
\[
O\!\left(\sum_{k=1}^K n_k m_k\right).
\]
When class sizes are roughly balanced, i.e.,
\[
n_k \approx n/K,
\qquad
m_k \approx m/K,
\]
this becomes
\[
O(nm/K).
\]
In addition, the classwise subproblems are independent and can therefore be parallelized straightforwardly.

\section{Proof of Proposition \ref{prop:kl_reweighting}}
\label{app:kl_reweighting}

\begin{proof}
	The optimization problem in \eqref{eq:kl_reweighting} is equivalent to
	\[
	\min_{w\in\Delta_n}
	\left\{
	\sum_{i=1}^n f_i^* w_i
	+
	\tau \sum_{i=1}^n w_i \log \frac{w_i}{w_i^0}
	\right\},
	\]
	where
	\[
	\Delta_n := \left\{ w\in\mathbb{R}_+^n : \sum_{i=1}^n w_i = 1 \right\}.
	\]

	Since \(w_i^0>0\) for all \(i\in[n]\) and \(\tau>0\), the objective is strictly convex on \(\Delta_n\), so the minimizer is unique. Moreover, the optimum is attained in the interior of \(\Delta_n\). Therefore, the first-order optimality conditions are necessary and sufficient for optimality.
	
	Consider the Lagrangian
	\[
	\mathcal{L}(w,\lambda)
	=
	\sum_{i=1}^n f_i^* w_i
	+
	\tau \sum_{i=1}^n w_i \log \frac{w_i}{w_i^0}
	+
	\lambda\left(\sum_{i=1}^n w_i-1\right).
	\]
	For each \(i\in[n]\), the first-order condition is
	\[
	\frac{\partial \mathcal{L}}{\partial w_i}
	=
	f_i^*
	+
	\tau\left(\log \frac{w_i}{w_i^0}+1\right)
	+
	\lambda
	=
	0.
	\]
	Hence
	\[
	w_i
	=
	w_i^0 \exp\!\left(-\frac{f_i^*}{\tau}\right)
	\exp\!\left(-\frac{\lambda}{\tau}-1\right).
	\]
	Summing over \(i\in[n]\), the constraint \(\sum_{i=1}^n w_i=1\) implies
	\[
	\exp\!\left(-\frac{\lambda}{\tau}-1\right)
	=
	\left(\sum_{j=1}^n w_j^0 \exp\!\left(-\frac{f_j^*}{\tau}\right)\right)^{-1}.
	\]
	Therefore,
	\[
	w_i
	=
	\frac{w_i^0 \exp(-f_i^*/\tau)}
	{\sum_{j=1}^n w_j^0 \exp(-f_j^*/\tau)},
	\qquad i\in[n].
	\]

	This proves that the unique minimizer \(\mu^*=\sum_{i=1}^n w_i^*\delta_{z_i}\) satisfies
	\[
	w_i^*
	=
	\frac{w_i^0 \exp(-f_i^*/\tau)}
	{\sum_{j=1}^n w_j^0 \exp(-f_j^*/\tau)},
	\qquad i\in[n].
	\]
\end{proof}

\paragraph{Scope of the theoretical results.}
Propositions~\ref{prop:class_conditioned_ot} and~\ref{prop:kl_reweighting}
characterize POTER's weight construction: they establish the classwise OT
decomposition under the stated mass-matching condition and the unique
KL-regularized reweighting solution for given dual potentials. These results
do not provide a finite-sample guarantee on the target risk of a classifier
trained using the resulting weights. Under label shift, where class proportions
change while class-conditional distributions remain fixed, establishing such a
guarantee would additionally require relating the population reweighted objective
to the target risk, controlling estimation error in the OT potentials and sample
weights, and analyzing the generalization of weighted ERM with these
data-dependent weights. We leave such an analysis to future work.

Known target class proportions can be incorporated into the class-mass
normalization step. When these proportions are unknown, our default
class-balancing scheme is motivated by worst-group robustness. For a fixed
classifier and fixed group-conditional distributions, the population worst-group
accuracy lower-bounds the accuracy under any mixture of the groups. Thus,
this bound continues to hold under changes in group composition, including
those induced by changes in class proportions when the groups jointly encode
class and spurious attributes.
This population-level relation motivates the use of worst-group accuracy when
the test-time group composition is unknown.

\section{Dataset Details}
\label{sec:appendix-dataset}
\textbf{Colored MNIST (CMNIST)} \cite{arjovsky2019invariant, yao2022improving}.
The CMNIST dataset is a binary classification benchmark derived from MNIST. Digits \(0\)--\(4\) are assigned to class \(0\), and digits \(5\)--\(9\) are assigned to class \(1\). Digit color is used as the spurious attribute, with two possible colors: red and green. This yields four groups defined by the pair of class label and color:
\[
g_1=\{0, \text{green}\},\quad
g_2=\{1, \text{green}\},\quad
g_3=\{0, \text{red}\},\quad
g_4=\{1, \text{red}\}.
\]
The training, validation, and test sets contain \(30{,}000\), \(10{,}000\), and \(20{,}000\) examples, respectively. In the training set, color is spuriously correlated with the label: the red-to-green ratio is \(8:2\) for class \(0\) and \(2:8\) for class \(1\). In the validation set, the color distribution is balanced within each class, with a red-to-green ratio of \(1:1\). In the test set, the label-color correlation is reversed: the red-to-green ratio is \(1:9\) for class \(0\) and \(9:1\) for class \(1\). Following \cite{arjovsky2019invariant}, labels are randomly flipped with probability \(0.25\).

\begin{figure}[htbp]
	\centering
	\subfigure[{0, green}]{
		\includegraphics[width=3cm, height=3.5cm, keepaspectratio]{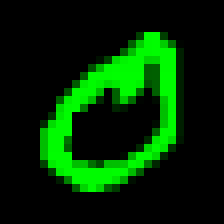}
	}
	\subfigure[{1, green}]{
		\includegraphics[width=3cm, height=3.5cm, keepaspectratio]{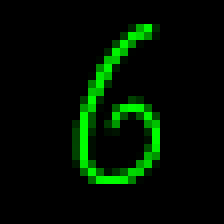}
	}
	\subfigure[{0, red}]{
		\includegraphics[width=3cm, height=3.5cm, keepaspectratio]{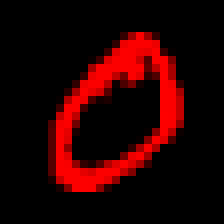}
	}
	\subfigure[{1, red}]{
		\includegraphics[width=3cm, height=3.5cm, keepaspectratio]{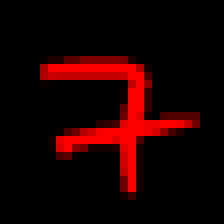}
	}
	\caption{Example images from the CMNIST dataset. The groups are $g_1 = \{0, \text{green}\}$, $g_2 = \{1, \text{green}\}$, $g_3 = \{0, \text{red}\}$, and $g_4 = \{1, \text{red}\}$.}
	\label{fig:cmnist_examples}
\end{figure}

\textbf{Waterbirds} \cite{sagawa2019distributionally}.
The Waterbirds dataset is a binary image classification benchmark constructed by placing bird images from the Caltech-UCSD Birds 200-2011 (CUB) dataset \cite{wah2011caltech} onto backgrounds from the Places dataset \cite{zhou2017places}. The target label is bird type, and the spurious attribute is background type. Bird species in CUB are partitioned into waterbirds and landbirds; the waterbird class includes Albatross, Auklet, Cormorant, Frigatebird, Fulmar, Gull, Jaeger, Kittiwake, Pelican, Puffin, Tern, Gadwall, Grebe, Mallard, Merganser, Guillemot, and Pacific Loon, while all remaining species are treated as landbirds. The four groups are defined by the pair of bird type and background:
\[
\begin{aligned}
g_1&=\{\text{landbird},\text{land}\},&
g_2&=\{\text{landbird},\text{water}\},\\
g_3&=\{\text{waterbird},\text{land}\},&
g_4&=\{\text{waterbird},\text{water}\}.
\end{aligned}
\]
The standard split contains \(4{,}795\) training examples, \(1{,}199\) validation examples, and \(5{,}794\) test examples. In the training set, bird type is strongly correlated with background: \(3{,}498\) examples are landbirds on land backgrounds, \(184\) are landbirds on water backgrounds, \(56\) are waterbirds on land backgrounds, and \(1{,}057\) are waterbirds on water backgrounds. Thus, the bias-conflicting minority groups are \(g_2\) and \(g_3\). The validation split contains \(467\), \(466\), \(133\), and \(133\) examples from \(g_1\) through \(g_4\), respectively.

\begin{figure}[htbp]
	\centering
	\subfigure[{landbird, land}]{
		\includegraphics[width=3cm, height=3.5cm, keepaspectratio]{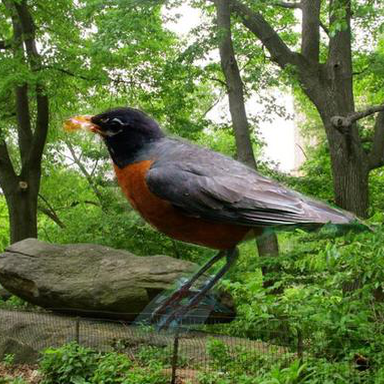}
	}
	\subfigure[{landbird, water}]{
		\includegraphics[width=3cm, height=3.5cm, keepaspectratio]{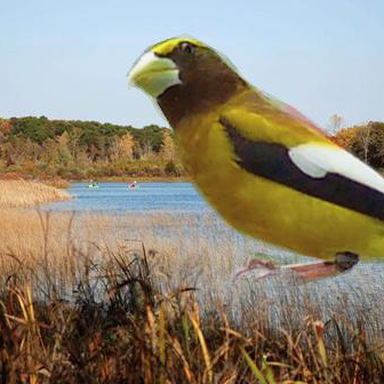}
	}
	\subfigure[{waterbird, land}]{
		\includegraphics[width=3cm, height=3.5cm, keepaspectratio]{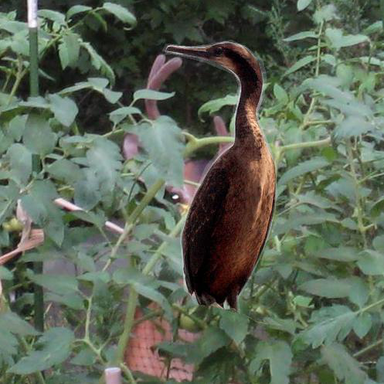}
	}
	\subfigure[{waterbird, water}]{
		\includegraphics[width=3cm, height=3.5cm, keepaspectratio]{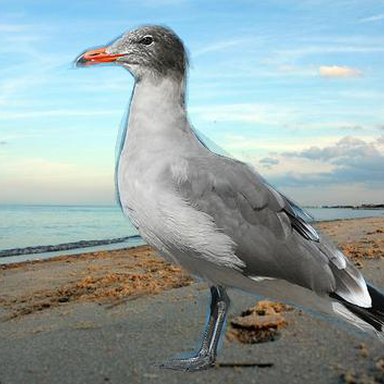}
	}
	\caption{Example images from the Waterbirds dataset. The groups are $g_1 = \{\text{landbird, land}\}$, $g_2 = \{\text{landbird, water}\}$, $g_3 = \{\text{waterbird, land}\}$, and $g_4 = \{\text{waterbird, water}\}$.}
	\label{fig:waterbirds_examples}
\end{figure}

\textbf{CelebA} \cite{liu2015deep}.
CelebA is a large-scale facial attribute dataset. We use the standard blond-hair classification task, where the target label is hair color and the spurious attribute is gender. The target label is binary, distinguishing blond hair from non-blond hair. The four groups are defined by the pair of hair color and gender:
\[
\begin{aligned}
g_1&=\{\text{non-blond hair},\text{female}\},&
g_2&=\{\text{non-blond hair},\text{male}\},\\
g_3&=\{\text{blond hair},\text{female}\},&
g_4&=\{\text{blond hair},\text{male}\}.
\end{aligned}
\]
The training and validation splits contain \(162{,}770\) and \(19{,}867\) examples, respectively. In the training set, the group counts for \(g_1\) through \(g_4\) are \(71{,}629\), \(66{,}874\), \(22{,}880\), and \(1{,}387\), respectively. The validation split contains \(8{,}535\), \(8{,}276\), \(2{,}874\), and \(182\) examples from \(g_1\) through \(g_4\), respectively. Thus, the blond male group \(g_4\) is the minority group and is substantially underrepresented compared with the other groups.

\begin{figure}[htbp]
	\centering
	\subfigure[{non-blond, female}]{
		\includegraphics[width=3cm, height=3.5cm, keepaspectratio]{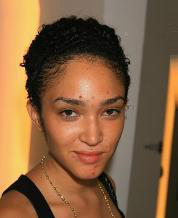}
	}
	\subfigure[{non-blond, male}]{
		\includegraphics[width=3cm, height=3.5cm, keepaspectratio]{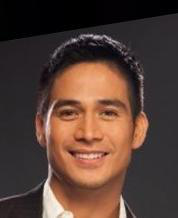}
	}
	\subfigure[{blond, female}]{
		\includegraphics[width=3cm, height=3.5cm, keepaspectratio]{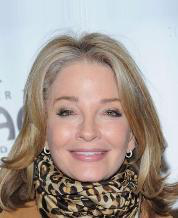}
	}
	\subfigure[{blond, male}]{
		\includegraphics[width=3cm, height=3.5cm, keepaspectratio]{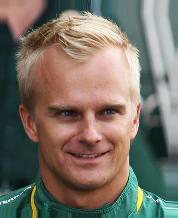}
	}
	\caption{Example images from the CelebA dataset. The groups are $g_1 = \{\text{non-blond hair, female}\}$, $g_2 = \{\text{non-blond hair, male}\}$, $g_3 = \{\text{blond hair, female}\}$, and $g_4 = \{\text{blond hair, male}\}$.}
	\label{fig:celeba_examples}
\end{figure}

\textbf{CivilComments} \cite{borkan2019nuanced}.
CivilComments is a text classification benchmark for online toxicity detection.
We use the WILDS \cite{koh2021wilds} version of the dataset, where the target
label is binary toxicity: \(Y=0\) denotes a non-toxic comment and \(Y=1\)
denotes a toxic comment. The WILDS dataset defines four coarse groups by the
combination of toxicity and whether a comment mentions any identity attribute.
The training, validation, and test splits contain \(269{,}038\), \(45{,}180\),
and \(133{,}782\) comments, respectively.

Following prior work, we use the 16-group evaluation protocol based on eight
demographic identity attributes:
\[
\mathcal{A}=\{\text{male},\ \text{female},\ \text{LGBTQ},\ \text{Christian},\
\text{Muslim},\ \text{other religion},\ \text{black},\ \text{white}\}.
\]
Let \(A_j(x)\in\{0,1\}\) indicate whether comment \(x\) mentions the \(j\)-th
identity attribute. For each identity attribute \(j\in\{1,\ldots,8\}\) and
toxicity label \(y\in\{0,1\}\), we define
\[
g_{j,y}
=
\{(x,y') : A_j(x)=1,\; y'=y\}.
\]
Thus, the evaluation groups correspond to toxic and non-toxic comments
mentioning each of the eight identity attributes. Unlike CMNIST, Waterbirds,
and CelebA, these 16 groups are overlapping rather than a partition of the
dataset, since a single comment may mention multiple identities. Per-group
sample counts across the train and validation splits are reported in
\autoref{tab:civilcomments_group_counts}.

\begin{table}[htbp]
\centering
\caption{Per-group sample counts in CivilComments-WILDS for each of the eight identity attributes. Counts are obtained by thresholding both the toxicity score and the identity-attribute annotation at \(0.5\). Groups are overlapping: a comment that mentions multiple identities is counted in multiple rows.}
\label{tab:civilcomments_group_counts}
\small
\setlength{\tabcolsep}{6pt}
\renewcommand{\arraystretch}{1.1}
\begin{tabular}{l rr rr}
\specialrule{1.2pt}{0pt}{0pt}
\multirow{2}{*}[-1.0ex]{Identity} &
\multicolumn{2}{c}{Train} &
\multicolumn{2}{c}{Validation} \\
\cmidrule(lr){2-3} \cmidrule(lr){4-5}
& \(Y=0\) & \(Y=1\) & \(Y=0\) & \(Y=1\) \\
\midrule
male            & 25{,}373 & 4{,}437 & 4{,}050 &   715 \\
female          & 31{,}282 & 4{,}962 & 5{,}120 &   771 \\
LGBTQ           &  6{,}155 & 2{,}265 & 1{,}099 &   358 \\
Christian       & 24{,}292 & 2{,}446 & 4{,}166 &   384 \\
Muslim          & 10{,}829 & 3{,}125 & 1{,}598 &   512 \\
other religion &  5{,}541 & 1{,}003 &    824 &   162 \\
black           &  6{,}785 & 3{,}111 & 1{,}119 &   533 \\
white           & 12{,}016 & 4{,}682 & 2{,}015 &   852 \\
\specialrule{1.2pt}{0pt}{0pt}
\end{tabular}
\end{table}

\section{Baseline Details}
\label{sec:appendix-baseline}

We provide brief descriptions of the baselines used in our comparisons.

\begin{itemize}[leftmargin=*]
	\item \textbf{ERM}:
	Empirical risk minimization trains the model by minimizing the average training loss, without using group information or an explicit robustness objective. It serves as the standard non-robust reference point.
	
	\item \textbf{Group DRO} \cite{sagawa2019distributionally}:
	Group DRO assumes access to group annotations during training and optimizes the worst-group loss rather than the average loss. It adaptively emphasizes groups with larger losses, encouraging the model to improve performance on underrepresented or bias-conflicting groups.
	
	\item \textbf{LISA} \cite{yao2022improving}:
	LISA mitigates reliance on spurious features through a selective interpolation strategy. It applies Mixup-style augmentation between example pairs that either share the same label but differ in the spurious attribute, or share the same spurious attribute but differ in the label. This selective pairing encourages the learned representation to be less dependent on shortcut correlations.
	
	\item \textbf{DFR} \cite{kirichenkolast}:
	DFR first trains a standard ERM model and then retrains only the last layer using a group-balanced subset. In our comparison, we follow the training-set variant used in prior work and construct the balanced retraining set from training examples with group annotations, denoted by DFR$^{\mathrm{Tr}}$.
	
	\item \textbf{PDE} \cite{deng2024robust}:
	PDE begins training from a group-balanced subset of the training data and progressively expands this subset to include additional training examples. This staged expansion is designed to delay the model's reliance on dominant spurious correlations while still allowing it to use more data as training proceeds.
	
    \item \textbf{JTT} \cite{liu2021just}:
    JTT is a two-stage approach that first trains an ERM model and uses misclassified examples as proxies for hard or bias-conflicting examples. It then constructs an upsampled training set in which these examples are repeated, and trains a second model on this dataset, improving robustness without requiring group annotations for the full training set.
	
	\item \textbf{CnC} \cite{zhang2022correct}:
	CnC is a two-stage approach that first trains a standard ERM model and uses its predictions to identify pairs of examples that share the same class but are likely to differ in spurious attributes. It then trains a second model with a contrastive objective that pulls together same-class examples with different spurious attributes, reducing dependence on shortcut cues.
	
	\item \textbf{SSA} \cite{nam2022spread}:
	SSA estimates spurious attributes for training examples using limited group supervision and forms pseudo groups from the predicted attributes and class labels. Group DRO is then applied to these inferred groups.
	
	\item \textbf{AFR} \cite{qiu2023simple}:
	AFR first trains a standard ERM model and then retrains the last layer with a weighted loss on top of fixed representations. The weighting rule emphasizes examples on which the ERM model performs poorly, increasing the influence of minority or bias-conflicting samples without relying on training-set group annotations.
	
    \item \textbf{GIC} \cite{hanimproving}:
    GIC infers group labels by training a spurious-attribute classifier using
    comparison data with a different group distribution from the training set. The
    predicted spurious attributes are combined with class labels to form inferred
    groups, which can then be used with downstream robust learning methods such as
    Group DRO.
    	
	\item \textbf{LEIA} \cite{gourabathina2026robustness}:
	LEIA proposes a two-stage adaptation method that identifies a low-dimensional subspace in the representation where model errors concentrate. It then applies a low-rank adjustment to the classifier logits, improving group robustness in a parameter-efficient manner without modifying the backbone or requiring full training-set group annotations.
	
	\item \textbf{KNN-RAD} \cite{NEURIPS2024_fba9133d}:
	KNN-RAD addresses label noise under spurious correlations by first performing kNN-based label correction in representation space and then applying a robust last-layer retraining procedure. It is the most closely related baseline in our label-noise experiments because it explicitly targets the joint setting of label noise and spurious correlations.
\end{itemize}

\section{Implementation Details}
\label{sec:appendix-implementation}

All experiments were conducted using a single NVIDIA GeForce RTX 3090 GPU.
We first describe the implementation details of POTER and then provide details
on baseline implementation.
The code for POTER is available at \url{https://github.com/Sung-Ho-Jo/POTER}.

\subsection{Feature Embeddings and Training Configurations}

For each dataset, we compute OT dual potentials in a fixed pretrained feature
space \(\phi(\cdot)\). For CMNIST, Waterbirds, and CelebA, we use features from
an ImageNet-pretrained ResNet-50. For CivilComments, we use features from
BERT-base-uncased. Importantly, these feature extractors are not trained on the
downstream task before computing OT dual potentials, distinguishing POTER from
methods that rely on representations learned by a preliminary task-specific
model. Features are extracted once before the OT step and kept fixed for OT
computation.

For the final reweighted ERM training stage, we use the same architecture as the
feature extractor used for OT: ResNet-50 for image datasets and BERT-base-uncased
for CivilComments. We use SGD for CMNIST, Waterbirds, and CelebA, and AdamW for
CivilComments. The training configurations are summarized in
Table~\ref{tab:implementation_training_configs}. These settings are based on
prior protocols for image datasets \cite{jo2025mitigating} and CivilComments
\cite{liu2021just}, with the exact values reported in the table. We do not use
a learning rate scheduler. For Waterbirds and CelebA, we use standard data
augmentations. Specifically, we apply random resized cropping, horizontal
flipping, and color jittering for Waterbirds, and random resized cropping,
horizontal flipping, and random rotation for CelebA. Model selection is
performed using worst-group accuracy on the validation set, and all reported
results are averaged over three random seeds.

\begin{table}[htbp]
\centering
\caption{Training configurations for the final reweighted ERM stage.}
\label{tab:implementation_training_configs}
\small
\setlength{\tabcolsep}{6pt}
\renewcommand{\arraystretch}{1.1}
\begin{tabular}{l c c c c c}
\specialrule{1.2pt}{0pt}{0pt}
Dataset & Optimizer & Learning Rate & Weight Decay & Batch Size & Epochs \\
\midrule
CMNIST        & SGD   & \(10^{-2}\) & \(10^{-2}\) & 128 & 300 \\
Waterbirds    & SGD   & \(10^{-5}\) & \(1.0\)      & 128 & 300 \\
CelebA        & SGD   & \(10^{-5}\) & \(10^{-2}\) & 128 & 50  \\
CivilComments & AdamW & \(10^{-5}\) & \(10^{-2}\) & 32  & 3   \\
\specialrule{1.2pt}{0pt}{0pt}
\end{tabular}
\end{table}

\subsection{Reference Distribution Construction}

For the standard benchmark setting and subgroup-concentrated label noise
setting, we construct the reference distribution from validation examples in the
bias-conflicting groups: landbird-on-water and waterbird-on-land for
Waterbirds, non-blond female and blond male for CelebA, and the two
label-color-conflicting groups for CMNIST. For CivilComments, following the
supervision setting used in prior work
\cite{liu2021just,hanimproving,gourabathina2026robustness}, we assume access
only to the binary identity-mention indicator rather than the full set of
identity-specific group annotations used for 16-group evaluation. Accordingly,
the reference distribution is constructed using the four groups defined by
toxicity and whether any identity attribute is mentioned; specifically, we use
the non-toxic identity-mentioning group and the toxic non-identity-mentioning
group.

For the symmetric label noise setting, we instead construct the reference
distribution by including both bias-aligned and bias-conflicting examples within
each class. Specifically, within each class, we subsample the bias-aligned and
bias-conflicting examples to have matched counts, so that the reference
distribution is balanced with respect to the spurious attribute conditional on
the class label. Within the selected reference set, we use uniform empirical
weights.

\subsection{Optimal Transport Computation}
\label{sec:appendix-ot-computation}

We compute OT dual potentials using class-conditioned optimal transport. For
each observed class, we solve an entropy-regularized OT problem between the
training examples and reference examples belonging to that class. We normalize
the training and reference empirical measures separately within each class.
For a class \(k\) with \(n_k\) training examples and \(m_k\) reference examples,
we assign uniform masses \(1/n_k\) and \(1/m_k\) to the training and reference
examples, respectively.
Both marginals therefore have total mass one, satisfying the classwise
mass-matching condition in Appendix~\ref{app:class_conditioned_ot} by
construction, even when the training and reference sample counts differ.

In the pretrained feature space defined by \(\phi(\cdot)\), we use the squared
Euclidean distance as the ground cost:
\[
C_{ij}
=
\|\phi(x_i)-\phi(x_j')\|_2^2,
\]
where \(x_i\) is a training example and \(x_j'\) is a reference example.

We solve each class-conditioned entropy-regularized OT problem using the
Sinkhorn algorithm~\cite{cuturi2013sinkhorn}. For each class \(c\), the entropy
regularization strength is set to
\[
\varepsilon_c = 0.03\,\bar C_c,
\]
where \(\bar C_c\) denotes the average cost within class \(c\). Sinkhorn is run
for at most 300 iterations with convergence threshold \(10^{-6}\). After
convergence, we obtain the training-side dual potentials \(f^*\) for each
class-conditioned OT problem. We then standardize the dual potentials within
each class using median-MAD normalization to align scales across classes, and
concatenate them into a single vector over the training set.

\subsection{Sample Reweighting}

Given the class-conditioned dual potentials \(f_i^*\), we instantiate the
closed-form reweighting rule in Proposition~\ref{prop:kl_reweighting}. In our
implementation, we set
\[
\tau = \alpha \cdot \mathrm{median}(f^*),
\]
where \(\alpha\) controls the sharpness of the weighting. As noted in
Section~\ref{sec:preliminaries}, for fixed \(\tau\), the reweighting rule is invariant to the
additive-constant ambiguity of OT dual potentials: for any \(a\in\mathbb{R}\),
\[
\frac{\exp(-(f_i^*+a)/\tau)}
{\sum_{j=1}^n \exp(-(f_j^*+a)/\tau)}
=
\frac{\exp(-f_i^*/\tau)}
{\sum_{j=1}^n \exp(-f_j^*/\tau)}.
\]
Therefore, dual potentials that differ only by an additive constant lead to the
same normalized sample weights.

We then apply additional normalization steps to the resulting weights. To
prevent class imbalance from dominating the reweighted training distribution,
we normalize the weights within each class and then rescale them so that each
class has equal total mass. The resulting weights are used for weighted sampling
during training. As a practical safeguard, we clip the maximum weight to prevent
the sampling distribution from concentrating too much mass on a small number of
examples.

We tune \(\alpha\) and the clipping threshold using validation worst-group
accuracy. For Waterbirds and CelebA, we search over
\(\alpha \in \{0.1,0.2\}\) and clipping thresholds in \(\{3,5,10,20\}\). For
CMNIST and CivilComments, we search over
\(\alpha \in \{0.4,0.5,0.6,0.7\}\) and clipping thresholds in
\(\{2,3,5,10\}\).

\paragraph{Choice of class masses.}
Class-mass normalization controls the relative contribution of
each observed class to the reweighted training objective while leaving the
classifier's hypothesis class unchanged. More generally, this step can assign
a prescribed fraction \(q_k\) of the total weight to class \(k\) before
clipping, where \(q_k\geq 0\) and \(\sum_{k=1}^K q_k=1\). Our default is
\(q_k=1/K\). Alternatives include known target class proportions or the
observed training class proportions, \(q_k=n_k/n\), where \(n_k\) denotes the
number of training examples with observed label \(k\). These alternatives can
be incorporated by changing the classwise rescaling targets.

The reference mixture weights \(\rho\) specify the composition
of \(\mu_{\mathrm{ref}}\), whereas \(q\) controls the class masses in the
subsequent weight-normalization step. In our experiments, the equal-class
default is motivated by worst-group robustness when the test-time group
composition is unknown, as discussed in Appendix~\ref{app:kl_reweighting}.

\subsection{Baseline Implementation}

For the standard benchmark results in Table~\ref{tab:main_results}, we report
the performance from the original papers whenever available, so that each
baseline is evaluated under its intended experimental protocol. When a result is
not reported in the original paper, we use the reported results from
\cite{hanimproving} or \cite{jo2025mitigating}. If the result is unavailable in
these sources as well, we reproduce the method ourselves following the
experimental setting described in the corresponding paper.

For the label-noise experiments in
Table~\ref{tab:waterbirds_celeba_label_noise_results}, we directly run all
baselines under the same label-noise protocol used for POTER, following the
implementation details and hyperparameter settings described in the original
papers. All results are averaged over three random seeds.

\section{Additional Benchmark Results}
\label{sec:appendix-additional-results}

We report detailed label-noise results and extended baseline comparisons
for the benchmarks studied in the main paper. We also evaluate POTER under
stronger spurious correlation on CMNIST.

\subsection{Full Results under Subgroup-Concentrated Label Noise}
\label{sec:appendix-subgroup-noise}

Table~\ref{tab:subgroup_label_noise_results} reports the full numerical results
corresponding to Figure~\ref{fig:subgroup_concentrated_label_noise}. Label
corruption is restricted to the landbird-on-water and waterbird-on-land groups
in Waterbirds and the blond-male group in CelebA. The reported noise rate is
the fraction of labels corrupted within each affected subgroup. Across both
datasets, POTER remains substantially more robust than KNN-RAD as the
corruption rate within these minority subgroups increases.

\begin{table}[htbp]
	\centering
    \caption{Worst-group accuracy on Waterbirds and CelebA under different subgroup-concentrated label noise rates.}
	\label{tab:subgroup_label_noise_results}
	\small
	\setlength{\tabcolsep}{8pt}
	\renewcommand{\arraystretch}{1.15}
	\resizebox{\textwidth}{!}{%
	\begin{tabular}{l l cccccc}
		\specialrule{1.2pt}{0pt}{0pt}
		\multirow{2}{*}[-1.0ex]{Dataset} &
		\multirow{2}{*}[-1.0ex]{Method} &
		\multicolumn{6}{c}{\raisebox{-0.6ex}{Label Noise (\%)}} \\
		\cmidrule(lr){3-8}
		& & 10 & 20 & 30 & 40 & 50 & 60 \\
		\midrule
        
		\multirow{2}{*}{Waterbirds}
		& KNN-RAD
		& 88.1$_{\pm 1.1}$ 
		& 77.5$_{\pm 4.3}$ 
		& 72.6$_{\pm 3.8}$ 
		& 52.9$_{\pm 5.8}$
		& 17.9$_{\pm 2.2}$
		& 7.5$_{\pm 0.2}$ \\
		
		& POTER (Ours)
        & \textbf{90.0}$_{\pm 0.1}$  
        & \textbf{89.0}$_{\pm 0.4}$  
        & \textbf{86.4}$_{\pm 0.4}$
		& \textbf{84.8}$_{\pm 0.3}$
		& \textbf{82.4}$_{\pm 0.3}$
		& \textbf{78.3}$_{\pm 0.9}$ \\
		
		\specialrule{0.8pt}{2pt}{2pt}
		
		\multirow{2}{*}{CelebA}
		& KNN-RAD
		& 79.4$_{\pm 0.0}$ 
		& 76.9$_{\pm 0.8}$
		& 73.3$_{\pm 0.6}$
		& 70.4$_{\pm 0.6}$
		& 73.0$_{\pm 0.3}$
		& 63.7$_{\pm 0.3}$ \\
		
		& POTER (Ours)
        & \textbf{90.6}$_{\pm 0.5}$  
        & \textbf{88.0}$_{\pm 0.9}$  
        & \textbf{87.6}$_{\pm 1.2}$
		& \textbf{87.2}$_{\pm 1.5}$
		& \textbf{87.2}$_{\pm 1.5}$
		& \textbf{84.8}$_{\pm 1.4}$ \\
		
		\specialrule{1.2pt}{0pt}{0pt}
	\end{tabular}
}
\end{table}

\subsection{Results under Symmetric Label Noise on CivilComments}
\label{sec:appendix-civilcomments-noise}

Table~\ref{tab:civilcomments_label_noise_results} reports additional symmetric
label noise results on CivilComments. CivilComments exhibits a different pattern from Waterbirds and CelebA, as the
16-group evaluation structure is less dominated by within-class
spurious-attribute imbalance. In Waterbirds and
CelebA, each class contains both bias-aligned and bias-conflicting groups with
highly imbalanced frequencies, so robustness requires identifying rare
bias-conflicting examples within the same class. In contrast, CivilComments
groups are defined by toxicity and identity mentions, and for every identity
attribute, non-toxic examples are substantially more frequent than toxic
examples (Table~\ref{tab:civilcomments_group_counts}). Thus, the dominant
imbalance in this evaluation is closer to class imbalance between toxic and
non-toxic comments than to a within-class spurious correlation structure.

This helps explain why AFR remains relatively robust on CivilComments. AFR
upweights examples according to a loss-based weighting parameter selected on the
validation set. Under label noise, aggressively upweighting high-loss examples
can hurt validation performance because mislabeled samples also tend to incur
large losses. Consequently, validation selection favors a near-uniform weighting
regime, making the second-stage training behave similarly to ERM rather than
strongly emphasizing high-loss examples. At the same time, AFR applies class
mass matching in the last-layer retraining stage, which balances the toxic and
non-toxic classes. This additional class balancing is particularly effective on
CivilComments, where the dominant imbalance across the 16 evaluation groups is
largely driven by the scarcity of toxic examples. This contrasts with
Waterbirds and CelebA, where class balancing alone cannot recover the rare
bias-conflicting groups within each class.

Overall, POTER, KNN-RAD, and AFR all remain relatively robust on CivilComments
under symmetric label noise. POTER achieves the best performance at 0\% and
10\% noise and remains competitive at higher noise rates, while the smaller gap
among these methods reflects the different nature of the CivilComments
evaluation compared with the image benchmarks.

\begin{table}[htbp]
	\centering
        \caption{Worst-group accuracy on CivilComments under different symmetric label noise rates.}
	\label{tab:civilcomments_label_noise_results}
	\small
	\setlength{\tabcolsep}{6pt}
	\renewcommand{\arraystretch}{1.15}
	\resizebox{\textwidth}{!}{%
	\begin{tabular}{l l c c cccc}
		\specialrule{1.2pt}{0pt}{0pt}
		\multirow{2}{*}[-1.0ex]{Dataset} &
		\multirow{2}{*}[-1.0ex]{Method} &
		\multirow{2}{*}[-1.0ex]{\makecell[c]{Group Labels\\(Tr/Val)}} &
		\multirow{2}{*}[-1.0ex]{\makecell[c]{No Extra\\Training}} &
		\multicolumn{4}{c}{\raisebox{-0.6ex}{Label Noise (\%)}} \\
		\cmidrule(lr){5-8}
		& & & & 0 & 10 & 20 & 30 \\
		\midrule
		
		\multirow{4}{*}{CivilComments}
        & Group DRO
        & Tr/Val & \checkmark
        & 69.0$_{\pm 0.9}$
        & 57.6$_{\pm 3.0}$
        & 57.2$_{\pm 3.6}$
        & 50.1$_{\pm 6.8}$  \\

        & AFR
        & Val & -
        & 66.2$_{\pm 1.8}$
        & 64.6$_{\pm 2.1}$
        & 65.4$_{\pm 1.4}$
        & 63.3$_{\pm 0.4}$ \\

		\cmidrule(lr){2-8}

		& KNN-RAD
		& Val & -
		& 69.6$_{\pm 0.0}$
		& 65.0$_{\pm 1.0}$
		& 62.7$_{\pm 1.4}$
		& 62.9$_{\pm 2.8}$\\
		
		& POTER (Ours)
		& Val & \checkmark
		& 71.0$_{\pm 0.6}$
        & 68.4$_{\pm 1.4}$  
        & 64.2$_{\pm 0.9}$  
        & 63.1$_{\pm 1.7}$   \\
		
		\specialrule{1.2pt}{0pt}{0pt}
	\end{tabular}
}
\end{table}

\subsection{Extended Baseline Comparison under Symmetric Label Noise}
\label{sec:appendix-extended-baselines}

Table~\ref{tab:additional_waterbirds_celeba_label_noise_results} provides an
extended comparison under symmetric label noise on Waterbirds and CelebA,
additionally including JTT and END~\cite{oh2022improving}. END is one of the
few methods that directly addresses spurious correlations under label noise.
As the original implementation of END is not publicly available, we report END
and JTT results as given in \cite{oh2022improving}.

The results reveal distinct robustness profiles across methods. JTT degrades
severely under label noise, consistent with the fragility of error-based sample
selection. END maintains relatively stable worst-group accuracy as the noise
rate increases, but incurs a substantial clean-setting cost: its clean-setting
worst-group accuracy is 82.8\% on Waterbirds and 82.6\% on CelebA, compared
with 90.9\% and 90.4\% for POTER, respectively. END also requires full-model
retraining. KNN-RAD outperforms loss-based baselines but exhibits noticeable
degradation at higher noise rates. POTER achieves the best or near-best
worst-group accuracy across all noise levels on both datasets, without
requiring any additional retraining stage.

\begin{table}[htbp]
	\centering
\caption{Extended comparison of worst-group accuracy on Waterbirds and CelebA under symmetric label noise. \emph{Extra Training} indicates the additional retraining required after base model training: \emph{Last} for last-layer retraining, \emph{Full} for full-model retraining, and \emph{Not required} for no additional retraining.}
	\label{tab:additional_waterbirds_celeba_label_noise_results}
    \footnotesize
	\setlength{\tabcolsep}{6pt}
    \renewcommand{\arraystretch}{1.08}
	\resizebox{\textwidth}{!}{%
	\begin{tabular}{l l c c cccc}
		\specialrule{1.2pt}{0pt}{0pt}
		\multirow{2}{*}[-0.8ex]{Dataset} &
		\multirow{2}{*}[-0.8ex]{Method} &
		\multirow{2}{*}[-0.8ex]{\makecell[c]{Group Labels\\(Tr/Val)}} &
		\multirow{2}{*}[-0.8ex]{\makecell[c]{Extra\\Training}} &
		\multicolumn{4}{c}{\raisebox{-0.5ex}{Label Noise (\%)}} \\
		\cmidrule(lr){5-8}
		& & & & 0 & 10 & 20 & 30 \\
		\midrule
		
		\multirow{6}{*}{Waterbirds}
        & Group DRO
        & Tr/Val & Not required
        & 90.6$_{\pm 0.2}$
        & 72.9$_{\pm 1.6}$
        & 54.3$_{\pm 1.0}$
        & 52.2$_{\pm 3.6}$ \\

        & JTT
        & Val & Full
        & 84.6$_{\pm 3.0}$
        & 56.5$_{\pm 8.0}$
        & 6.0$_{\pm 3.0}$
        & 2.7$_{\pm 1.0}$ \\

        & AFR
        & Val & Last
        & 88.3$_{\pm 0.9}$
        & 58.7$_{\pm 0.4}$
        & 61.2$_{\pm 5.9}$
        & 52.9$_{\pm 0.0}$ \\

		\cmidrule(lr){2-8}

        & END
        & Val & Full
        & 82.8$_{\pm 1.0}$
        & 84.2$_{\pm 1.0}$
        & 83.2$_{\pm 1.0}$
        & 81.8$_{\pm 1.0}$ \\

		& KNN-RAD
		& Val & Last
		& \textbf{91.0}$_{\pm 0.1}$
		& 82.4$_{\pm 0.7}$
		& 74.7$_{\pm 1.1}$
		& 68.9$_{\pm 2.1}$ \\
		
		& POTER (Ours)
		& Val & Not required
		& \cellcolor{yellow!12}90.9$_{\pm 0.3}$ 
        & \cellcolor{yellow!12}\textbf{89.4}$_{\pm 0.4}$  
        & \cellcolor{yellow!12}\textbf{87.9}$_{\pm 1.6}$  
        & \cellcolor{yellow!12}\textbf{85.3}$_{\pm 1.4}$ \\

		\specialrule{0.8pt}{1pt}{1pt}

		\multirow{6}{*}{CelebA}
		& Group DRO
		& Tr/Val & Not required
		& 89.3$_{\pm 1.3}$
		& 66.9$_{\pm 0.3}$
		& 59.5$_{\pm 2.6}$
		& 54.8$_{\pm 1.8}$ \\

        & JTT
        & Val & Full
        & 82.2$_{\pm 2.0}$
        & 74.8$_{\pm 2.0}$
        & 24.5$_{\pm 36.0}$
        & 15.1$_{\pm 16.0}$ \\
        
        & AFR
        & Val & Last
        & 80.9$_{\pm 2.3}$
        & 69.4$_{\pm 4.8}$
        & 67.7$_{\pm 2.4}$
        & 58.8$_{\pm 6.3}$ \\

		\cmidrule(lr){2-8}

        & END
        & Val & Full
        & 82.6$_{\pm 2.0}$
        & 79.7$_{\pm 1.0}$
        & 81.1$_{\pm 2.0}$
        & 77.8$_{\pm 3.0}$ \\

		& KNN-RAD
		& Val & Last
		& 83.9$_{\pm 0.0}$
		& 80.0$_{\pm 0.9}$
		& 78.1$_{\pm 1.3}$
		& 70.2$_{\pm 2.6}$ \\
      
		& POTER (Ours)
		& Val & Not required
		& \cellcolor{yellow!12}\textbf{90.4}$_{\pm 0.5}$ 
		& \cellcolor{yellow!12}\textbf{87.6}$_{\pm 0.3}$ 
		& \cellcolor{yellow!12}\textbf{82.8}$_{\pm 1.0}$ 
		& \cellcolor{yellow!12}\textbf{78.2}$_{\pm 0.4}$ \\
		
		\specialrule{1.2pt}{0pt}{0pt}
	\end{tabular}
}
\end{table}

%%%%%%%%%%%%%%%%%%%%%%%%%%%%%%%%%%%%%%%%%%%%%%%%%%%%%%%%%%%%

\subsection{Results under Stronger Spurious Correlation on CMNIST}
\label{sec:appendix-cmnist-stronger-correlation}

We evaluate POTER as bias-conflicting training examples become increasingly
scarce, reducing their proportion from 20\% to 10\%, 5\%, and 1\%.
To isolate the effect of spurious-correlation strength, we remove label noise
throughout the dataset and keep the remaining experimental settings fixed.
Table~\ref{tab:cmnist_stronger_correlation} highlights POTER's robustness to
increasingly strong spurious correlations: it maintains 95.1\% worst-group
accuracy even when bias-conflicting examples constitute only 1\% of the
training data.

\begin{table}[htbp]
\centering
\caption{Worst-group accuracy (\%) on CMNIST under stronger spurious correlation.
Label noise is removed throughout the dataset for this experiment.}
\label{tab:cmnist_stronger_correlation}
\small
\setlength{\tabcolsep}{7pt}
\renewcommand{\arraystretch}{1.1}
\begin{tabular*}{\linewidth}{@{\extracolsep{\fill}}lcccc}
\specialrule{1.2pt}{0pt}{0pt}
Bias-conflicting training ratio & 20\% & 10\% & 5\% & 1\% \\
\midrule
POTER & \(98.4 \pm 0.2\) & \(98.0 \pm 0.2\) & \(97.3 \pm 0.3\) & \(95.1 \pm 0.2\) \\
\specialrule{1.2pt}{0pt}{0pt}
\end{tabular*}
\end{table}

\section{Empirical Analysis of POTER}
\label{sec:appendix-framework-analysis}

We examine how POTER's performance depends on validation supervision, the
feature extractor, and the choice of weighting method. We also measure the
computational overhead of weight construction.

\subsection{Sensitivity to Validation Supervision}
\label{sec:appendix-validation-robustness}

POTER uses a group-annotated validation set for reference construction and
model selection. The validation-supervised baselines in our comparison also
rely on such supervision. Nevertheless, because the reference distribution is
central to POTER's sample weighting, we assess its robustness when validation
supervision is limited or imperfect. Specifically, we vary the reference size,
reference-label noise, and availability of group annotations on Waterbirds.

\paragraph{Reduced reference size.}
We subsample the standard reference set of 599 examples to 25\% and 10\% of
its original size. As shown in Table~\ref{tab:reference_size}, POTER achieves
\(89.8 \pm 0.1\%\) worst-group accuracy with roughly 60 reference examples,
about 1\% of the combined training and validation data. This indicates that
a small group-annotated reference can provide a useful weighting signal.

\begin{table}[htbp]
\centering
\caption{Worst-group accuracy (\%) on Waterbirds when subsampling the reference
set. The full reference contains 599 examples; all other settings are unchanged.}
\label{tab:reference_size}
\small
\setlength{\tabcolsep}{10pt}
\renewcommand{\arraystretch}{1.1}
\begin{tabular*}{\linewidth}{@{\extracolsep{\fill}}lccc}
\specialrule{1.2pt}{0pt}{0pt}
Reference size & 100\% & 25\% & 10\% \\
\midrule
POTER & \(90.9 \pm 0.3\) & \(91.1 \pm 0.1\) & \(89.8 \pm 0.1\) \\
\specialrule{1.2pt}{0pt}{0pt}
\end{tabular*}
\end{table}

\paragraph{Noisy reference labels.}
We inject symmetric noise into the reference class labels at nominal rates
of 10\%, 20\%, and 30\%. Table~\ref{tab:reference_label_noise} shows a gradual
decrease in worst-group accuracy, from \(90.9 \pm 0.3\%\) with clean reference
labels to \(82.5 \pm 1.9\%\) at 30\% noise. Because the reference is
class-imbalanced, approximately 60\% of the reference examples labeled as
waterbirds are actually landbirds at 30\% nominal noise. This indicates that
POTER can retain strong worst-group performance even when most reference
examples assigned to one class are mislabeled.

\begin{table}[htbp]
\centering
\caption{Worst-group accuracy (\%) on Waterbirds under symmetric corruption of
reference class labels. At 30\% nominal noise, the observed-waterbird reference
has approximately 60\% contamination due to class imbalance.}
\label{tab:reference_label_noise}
\small
\setlength{\tabcolsep}{7pt}
\renewcommand{\arraystretch}{1.1}
\begin{tabular*}{\linewidth}{@{\extracolsep{\fill}}lcccc}
\specialrule{1.2pt}{0pt}{0pt}
Reference-label noise rate & 0\% & 10\% & 20\% & 30\% \\
\midrule
POTER & \(90.9 \pm 0.3\) & \(87.9 \pm 1.2\) & \(85.7 \pm 0.3\) & \(82.5 \pm 1.9\) \\
\specialrule{1.2pt}{0pt}{0pt}
\end{tabular*}
\end{table}

\paragraph{Without group annotations.}
We use the entire class-labeled validation set as the reference, without
group-based selection or balancing, and select models and hyperparameters
using mean validation accuracy. Group annotations are therefore unused in
reference construction, training, and model selection; test group labels
are used only to evaluate worst-group accuracy.
Table~\ref{tab:group_annotation_free} reports \(89.3 \pm 0.5\%\) worst-group
accuracy for this variant, compared with \(90.9 \pm 0.3\%\) for standard
POTER. This result likely benefits from the relatively balanced coverage of
groups in Waterbirds' validation split. This suggests that POTER can achieve
strong worst-group performance without validation group annotations when
the validation distribution is representative of the test distribution.

\begin{table}[htbp]
\centering
\caption{POTER with and without validation group annotations on Waterbirds.}
\label{tab:group_annotation_free}
\small
\setlength{\tabcolsep}{4pt}
\renewcommand{\arraystretch}{1.1}
\begin{tabular*}{\linewidth}{@{\extracolsep{\fill}}lllc}
\specialrule{1.2pt}{0pt}{0pt}
Variant & Reference & Selection metric & Worst-group (\%) \\
\midrule
Standard POTER & Group-selected subset & Worst-group accuracy & \(90.9 \pm 0.3\) \\
Without group annotations & Entire validation set & Mean accuracy & \(89.3 \pm 0.5\) \\
\specialrule{1.2pt}{0pt}{0pt}
\end{tabular*}
\end{table}

\subsection{Sensitivity to the Feature Extractor}
\label{sec:appendix-feature-sensitivity}

We vary the feature extractor used to compute the sample weights on Waterbirds,
keeping the remaining pipeline fixed and performing no extractor-specific
retuning. In addition to the default ImageNet-pretrained ResNet-50, we consider
a ResNet-50 trained by ERM on Waterbirds for 300 epochs and SigLIP~2 ViT-B/16,
which uses a different architecture and pretraining paradigm.

Table~\ref{tab:feature_extractor_sensitivity} reports mean sample weights within
each group and classification accuracy. Across all three extractors, each
bias-conflicting group receives a larger mean weight than its bias-aligned
counterpart within the same class: \(g_2\) versus \(g_1\), and \(g_3\) versus
\(g_4\). ERM-trained features yield worst-group accuracy comparable to the
default features (\(91.0\%\) versus \(90.9\%\)). SigLIP~2 also achieves strong
worst-group accuracy (\(87.9\%\)) and the highest average accuracy (\(95.2\%\)).
These results suggest that POTER is robust to the choice of feature extractor,
with consistent weighting behavior and strong performance across the tested
representations without extractor-specific retuning.

\begin{table}[htbp]
\centering
\caption{Feature-extractor sensitivity on Waterbirds without extractor-specific
retuning. Columns \(g_1\)--\(g_4\) report the mean sample weight for each group
defined in Appendix~\ref{sec:appendix-dataset}; \(g_2\) and \(g_3\) are
bias-conflicting groups.}
\label{tab:feature_extractor_sensitivity}
\small
\setlength{\tabcolsep}{4pt}
\renewcommand{\arraystretch}{1.1}
\begin{tabular*}{\linewidth}{@{\extracolsep{\fill}}lcccccc}
\specialrule{1.2pt}{0pt}{0pt}
Extractor & \(g_1\) & \(g_2\) & \(g_3\) & \(g_4\) & Worst-group (\%) & Average (\%) \\
\midrule
Pretrained ResNet-50 & 0.31 & 6.32 & 10.19 & 1.87 & \(90.9 \pm 0.3\) & \(92.7 \pm 0.1\) \\
ERM-trained ResNet-50 & 0.37 & 7.79 & 12.43 & 1.30 & \(91.0 \pm 0.1\) & \(91.3 \pm 0.0\) \\
SigLIP~2 ViT-B/16 & 0.25 & 6.89 & 11.24 & 1.92 & \(87.9 \pm 0.1\) & \(95.2 \pm 0.2\) \\
\specialrule{1.2pt}{0pt}{0pt}
\end{tabular*}
\end{table}

\subsection{Comparison with Distance-Based Weighting}
\label{sec:appendix-distance-weighting}

To isolate the contribution of OT, we replace its dual potentials with
class-conditioned \(k\)NN distances to the same reference set on Waterbirds.
We keep the feature representation, reference set, score normalization,
downstream training procedure, and validation-based tuning protocol fixed.
For the distance-based alternative, we jointly search
\(k \in \{1,5,10,20\}\) and a weight-transformation grid that includes and
extends the POTER grid. The selected configuration uses \(k=5\).

Table~\ref{tab:distance_based_weighting} shows that both methods assign higher
mean weights to the minority group within each class, suggesting that the
shared feature representation contains relevant subgroup information.
However, OT produces
substantially larger differences in mean weights between minority and majority
groups, indicating stronger relative emphasis on bias-conflicting samples.
POTER achieves \(90.9 \pm 0.3\%\) worst-group accuracy versus
\(76.7 \pm 0.6\%\) for \(k\)NN, a difference of 14.2 percentage points.
This controlled comparison suggests that reference proximity alone is
insufficient to adequately mitigate spurious correlations in this setting.
This limitation arises because \(k\)NN computes reference distances
independently for each training sample, allowing many samples to receive high
weights from the same small reference neighborhood without accounting for
local overrepresentation. For example, a densely populated training region
can remain close to a few reference points even when its share of training
mass greatly exceeds the reference mass in that region.
In contrast, OT's fixed marginal constraints make
these samples compete for limited reference mass, so its dual potentials
reflect regional excess or deficit of training mass.

\begin{table}[htbp]
\centering
\caption{Comparison of OT-based weighting and class-conditioned \(k\)NN distance weighting on Waterbirds.}
\label{tab:distance_based_weighting}
\small
\setlength{\tabcolsep}{6pt}
\renewcommand{\arraystretch}{1.1}
\begin{tabular*}{\linewidth}{@{\extracolsep{\fill}}lccccc}
\specialrule{1.2pt}{0pt}{0pt}
Weighting method & \(g_1\) & \(g_2\) & \(g_3\) & \(g_4\) & Worst-group (\%) \\
\midrule
Class-conditioned \(k\)NN weighting & 0.66 & 1.24 & 3.84 & 1.94 & \(76.7 \pm 0.6\) \\
OT-based weighting (POTER) & 0.31 & 6.32 & 10.19 & 1.87 & \(90.9 \pm 0.3\) \\
\specialrule{1.2pt}{0pt}{0pt}
\end{tabular*}
\end{table}

\subsection{Runtime Analysis and Comparison with SELF}
\label{sec:appendix-computational-overhead}

\paragraph{Runtime for computing OT dual potentials.}
Table~\ref{tab:ot_runtime} reports the wall-clock time required to compute the
OT dual potentials \(f^*\) after feature embeddings have been extracted.
Feature extraction and model training are excluded from these measurements.
The OT computation takes at most \(4.52\) seconds on average across the
evaluated datasets.

In our implementation, the downstream reweighting objective depends only on the
training-side dual potentials, not on the full transport plan. We therefore do
not store the full transport plan and retain only the Sinkhorn scaling variables
needed to compute \(f^*\). This reduces memory usage and makes the OT step
practical even for larger datasets.

\begin{table}[htbp]
\centering
\caption{Runtime for computing OT dual potentials \(f^*\) after feature embeddings
have been extracted. Results are averaged over five runs.}
\label{tab:ot_runtime}
\small
\setlength{\tabcolsep}{8pt}
\renewcommand{\arraystretch}{1.1}
\begin{tabular*}{\linewidth}{@{\extracolsep{\fill}}lrrrc}
\specialrule{1.2pt}{0pt}{0pt}
Dataset & \(n_{\mathrm{train}}\) & \(n_{\mathrm{ref}}\) & Dimension & Time (s) \\
\midrule
CMNIST        & 30{,}000  & 5{,}022 & 2{,}048 & \(0.41 \pm 0.00\) \\
Waterbirds    & 4{,}795   & 599   & 2{,}048 & \(0.07 \pm 0.00\) \\
CelebA        & 162{,}770 & 8{,}717 & 2{,}048 & \(4.52 \pm 0.03\) \\
CivilComments & 269{,}038 & 7{,}111 & 768 & \(4.31 \pm 0.00\) \\
\specialrule{1.2pt}{0pt}{0pt}
\end{tabular*}
\end{table}

\paragraph{Comparison with SELF.}
To assess POTER's computational overhead relative to last-layer retraining,
we compare it with SELF~\cite{labonte2024towards} on Waterbirds.
POTER requires a single model-training stage: reference construction, OT
computation, and weight transformation involve no additional model training.
Both pipelines include feature extraction and full-model training, so we
compare their additional method-specific operations.
Table~\ref{tab:self_overhead} reports the runtime of POTER's complete weight
construction and SELF's last-layer retraining. POTER's complete weight
construction takes 0.075\,s, while SELF's last-layer retraining takes 5.657\,s.
Although both overheads are small relative to full-model training, POTER
requires less additional computation and retains \(87.9 \pm 1.6\%\)
worst-group accuracy under 20\% symmetric label noise, compared with
\(14.7 \pm 0.2\%\) for SELF. This contrast highlights the fragility of
loss-based sample selection under label noise and the robustness of
POTER's loss-independent reweighting.

\begin{table}[htbp]
\centering
\caption{Comparison of POTER and SELF on Waterbirds: additional computation
and robustness to label noise. Worst-group accuracy is measured under 20\% symmetric label noise.}
\label{tab:self_overhead}
\small
\setlength{\tabcolsep}{5pt}
\renewcommand{\arraystretch}{1.1}
\begin{tabular*}{\linewidth}{@{\extracolsep{\fill}}llrc}
\specialrule{1.2pt}{0pt}{0pt}
Method & Timed operation & Time (s) & Worst-group (\%) \\
\midrule
SELF & Last-layer retraining & 5.657 & \(14.7 \pm 0.2\) \\
POTER & Complete weight construction & 0.075 & \(87.9 \pm 1.6\) \\
\specialrule{1.2pt}{0pt}{0pt}
\end{tabular*}
\end{table}

\section{Relationship to Prior Work}
\label{sec:appendix-prior-work}

\begingroup
\paragraph{Relationship to OT-based data valuation.}
POTER builds on the sensitivity interpretation of OT dual potentials used in
prior work such as LAVA~\cite{just2023lava}: the training-side potential
describes how the transport discrepancy changes with the mass assigned to
an individual training sample. LAVA uses scores derived from these potentials
to value training samples relative to a validation distribution, with
applications including the detection of noisy, mislabeled, and poisoned
samples. POTER adopts this interpretation and constructs a reference
distribution from group-annotated validation data to encode a desired group
composition, with greater emphasis on bias-conflicting groups. Under fixed
transport marginals, samples in an overrepresented training region compete
for limited reference mass. The class-conditioned potentials can therefore be
interpreted as reflecting local excess or deficit of training mass relative to
this deliberately shifted reference. In this view, even a correctly labeled
sample can receive a lower weight when its region is overrepresented relative
to the reference.
We convert these
potentials into sample weights through the KL-regularized reweighting rule in
Proposition~\ref{prop:kl_reweighting} and use the resulting weights for model
training. This formulation allows the same weighting rule to emphasize
bias-conflicting samples and downweight mislabeled samples that are
misaligned with the reference for their observed class, as illustrated in
Figure~\ref{fig:waterbird_analysis}. The central contribution is to develop
this view of sample importance into a loss-independent reweighting framework
for supervised learning: class-conditioned OT dual potentials provide
sample-level signals of relative mass concentration and determine each
sample's contribution to the training objective. This directly connects the
allocation of training mass to robust learning under spurious correlations
and label noise.

\paragraph{Relationship to reweighting methods for robust learning.}
Group DRO~\cite{sagawa2019distributionally} optimizes group losses using a
shared weighting coefficient for samples within each group. Methods such as
SSA~\cite{nam2022spread} and GIC~\cite{hanimproving} first infer group labels
and then use these labels in a group-robust training procedure. When combined
with Group DRO, their weighting likewise operates at the group level and
does not distinguish individual samples within a group. POTER retains
sample-level variation within each group, as illustrated in
Figure~\ref{fig:waterbird_analysis} and the accompanying analysis of the sample
weights. A different family of
methods uses a preliminary model to determine which samples to emphasize:
AFR~\cite{qiu2023simple} and LEIA~\cite{gourabathina2026robustness} assign
greater importance to examples on which the model performs poorly, while
SELF~\cite{labonte2024towards} selects examples using misclassification or
disagreement between models. Criteria based on high loss or misclassification
may confound bias-conflicting examples with mislabeled ones. The performance
degradation under label noise in
Tables~\ref{tab:waterbirds_celeba_label_noise_results},
\ref{tab:additional_waterbirds_celeba_label_noise_results}, and
\ref{tab:self_overhead} is consistent with this limitation.
POTER instead derives individual sample weights from class-conditioned
reference alignment, without requiring training-set group annotations or
losses from a preliminary classifier. It thus combines sample-level
differentiation with loss-independent weighting for robust learning.
\par
\endgroup

\end{document}